\documentclass[12pt,A4,]{article}
\usepackage[margin=1in,footskip=0.25in]{geometry}

\RequirePackage{fix-cm}

\usepackage{graphicx}
\usepackage{amsmath}
\usepackage{subfigure}
\usepackage{natbib}
\usepackage{bm}
\usepackage{amssymb}
\usepackage[ruled]{algorithm2e}
\usepackage{booktabs}
\usepackage{multirow}
\usepackage{array}
\usepackage{url}
\usepackage{float}

\usepackage[utf8]{inputenc}

\newtheorem{mycaption}{Figure}

\newtheorem{lemma}{Lemma}
\newtheorem{theorem}{Theorem}

\newtheorem{remark}{Remark}
\newtheorem{proof}{Proof}
\newcommand*{\qed}{\hfill\ensuremath{\square}}%

\usepackage{color}
\newcommand{\revise}[1]{\textcolor{black}{{}#1}}
\newcommand{\reviseb}[1]{\textcolor{black}{{}#1}}

\begin{document}

	\title{\bf Adaptive Iterative Hessian Sketch via $A$-Optimal Subsampling}
	
	\author{Aijun Zhang$^*$, Hengtao Zhang and Guosheng Yin\\
		{\normalsize  Department of Statistics and Actuarial Science, The University of Hong Kong}\\
		{\normalsize Pokfulam Road, Hong Kong}}
	\date{}

\maketitle

\begin{abstract}
Iterative Hessian sketch (IHS) is an effective sketching method for modeling large-scale data. It was originally proposed by  Pilanci and Wainwright (2016; JMLR) based on randomized sketching matrices. However, it is computationally intensive due to the iterative sketch process. In this paper, we analyze the IHS algorithm under the unconstrained least squares problem setting, then propose a deterministic approach for improving IHS via $A$-optimal subsampling. Our contributions are three-fold: (1) a good initial estimator based on the $A$-optimal design is suggested; (2) a novel ridged preconditioner is developed for repeated sketching; and (3) an exact line search method is proposed for determining the optimal step length adaptively. Extensive experimental results demonstrate that our proposed $A$-optimal IHS algorithm outperforms the existing accelerated IHS methods.
\end{abstract}

\section{Introduction}
\label{intro}

Consider the linear model $\bm{y}=\bm{X}\bm{\beta}+\bm{\varepsilon}$ with the response vector $\bm{y}\in\mathbb{R}^n$,  the design matrix $\bm{X} \in \mathbb{R}^{n\times d}$  and the noise term \revise{satisfying $\mathbb{E}(\bm{\varepsilon})=\bm{0}$ and $\mathbb{V}(\bm{\varepsilon})=\sigma^2\bm{I}_n$}. The unknown parameter $\bm\beta\in \mathbb{R}^d$ can be efficiently estimated by the method of least squares:
\begin{eqnarray}\label{LS}
\hat{\bm{\beta}}^{\rm LS} & = & \arg \min\limits_{\bm{\beta}} \frac{1}{2}\|\bm{X\beta}-\bm{y}\|_2^2\nonumber\\
&=& (\bm{X}^T\bm{X})^{-1}\bm{X}^T\bm{y},
\end{eqnarray}
which is of computational complexity $O(nd^2 + d^3)$. \revise{For the massive data with $n\gg d$, the least squares method would become computationally expensive, and it may  exceed the computing capacity.}

Faster least squares approximation can be achieved by the randomized sketch proposed by \cite{drineas2011faster}. It relies on a proper random matrix $\bm{S}\in\mathbb{R}^{m\times n}$ with $m\ll n$. The ways to generate random matrix $\bm{S}$ are divided into three main categories: subsampling, random projection and their hybrid.  A widely known subsampling method is \revise{based on statistical leverage (LEV) scores \citep{drineas2006sampling,mahoney2011randomized,drineas2012fast,ma2015statistical}. } 
The random projection approach includes the subsampled randomized Hadamard transformation (SRHT)  \citep{drineas2011faster,boutsidis2013improved} and the Clarkson--Woodruff sketch \citep{clarkson2013low}. The hybrid approach is to combine subsampling and random projection,  see e.g. \citet{mcwilliams2014fast}.  

When a sketching matrix $\bm{S}$ is fixed, there exist different sketching schemes, including the classical sketch (CS), the Hessian sketch (HS) and the iterative Hessian sketch (IHS). The widely adopted CS  uses the  sketched data pair $(\bm{SX}, \bm{Sy})$ for approximating   $\hat{\bm{\beta}}^{\rm LS}$ by 
\begin{eqnarray*}
\hat{\bm{\beta}}^{\rm CS} & =& \arg \min\limits_{\bm{\beta}} \frac{1}{2}\|\bm{SX\beta}-\bm{Sy}\|_2^2\nonumber\\
&=&(\bm{X}^T\bm{S}^T\bm{S}\bm{X})^{-1}\bm{X}^T\bm{S}^T\bm{S}\bm{y}.
\end{eqnarray*}
\citet{pilanci2016iterative}  showed that $\hat{\bm{\beta}}^{\rm CS}$ is suboptimal in the sense that it has a substantially larger error than $\hat{\bm{\beta}}^{\rm LS}$ with respect to the ground truth $\bm{\beta}^{*}$. 
They introduced the HS estimator based on the partially sketched data $(\bm{SX}, \bm{y})$,
\begin{eqnarray*}
\hat{\bm{\beta}}^{\rm HS} &=& \arg \min\limits_{\bm{\beta}} \frac{1}{2}\|\bm{SX\beta}\|_2^2-\bm{\beta}^T\bm{X}^T\bm{y}\nonumber\\
&=&(\bm{X}^T\bm{S}^T\bm{S}\bm{X})^{-1}\bm{X}^T\bm{y}
\end{eqnarray*}
and furthermore the IHS estimator based on iterative sketched data $\{\bm{S}_t\bm{X}, t=1,\ldots,N\}$,
\begin{eqnarray}\label{IHS}
\hat{\bm{\beta}}^{\rm IHS}_t &=& \arg \min\limits_{\bm{\beta}} \frac{1}{2}\|\bm{S}_t\bm{X}(\bm{\beta}-\hat{\bm{\beta}}^{\rm IHS}_{t-1})\|_2^2-\bm{\beta}^T\bm{X}^T\bm{e}_{t-1}
\nonumber\\
&=&\hat{\bm{\beta}}^{\rm IHS}_{t-1}+(\bm{X}^T\bm{S}^T_t\bm{S}_t\bm{X})^{-1}\bm{X}^T\bm{e}_{t-1},
\end{eqnarray}
where $\bm{e}_t=\bm{y}-\bm{X}\hat{\bm{\beta}}^{\rm IHS}_{t}$, with the initial $\hat{\bm{\beta}}^{\rm IHS}_0$ being provided. Unlike the CS and HS estimators, the IHS estimator is guaranteed to converge to $\bm{\hat\beta}^{\rm LS}$  upon some {\em good event} conditions \citep{pilanci2016iterative}.

The IHS can be interpreted as a first-order gradient descent method with a series of random preconditioners $\bm{M}_t=\bm{X}^T\bm{S}_{t}^T\bm{S}_{t}\bm{X}$ subject to the unit step length. The preconditioner is  widely used to boost optimization algorithms  \citep{knyazev2007steepest,gonen2016solving}. 
However, the IHS is computationally intensive since in every iteration $\bm{M}_t^{-1}$ has to be evaluated for a new random sketching matrix $\bm{S}_t$. To speed up the IHS,  \citet{wang2018large} proposed the pwGradient method to improve the IHS by a fixed well-designed sketching matrix, in which case the IHS reduces to a first order method with a constant preconditioner.  Meanwhile,  \citet{wang2017sketching} proposed another accelerated version of IHS with the fixed sketching matrix, while adopting conjugate gradient descent.  Note that all these sketch methods are based on randomized sketching matrices. 

In this paper we propose to reformulate the IHS estimator $\hat{\bm{\beta}}^{\rm IHS}_t$ by a linear combination of the initial $\hat{\bm{\beta}}^{\rm IHS}_0$ and the full data estimator $\hat{\bm{\beta}}^{\rm LS}_0$. Such reformulation enables us to find a sufficient isometric condition on the sketching matrices so that $\hat{\bm{\beta}}^{\rm IHS}_t$ is guaranteed to converge to $\bm{\hat\beta}^{\rm LS}$ with \revise{geometric convergence}. It then motivates us to propose a deterministic approach for improving the IHS based on $A$-optimal subsampling and adaptive step lengths, which modifies the original second-order IHS method to be an adaptive first-order method. In summary, we improve the IHS method with the following  three contributions: 
\begin{itemize}
	\item {\bf Good initialization.} A good initialization scheme can reduce the inner iteration rounds of IHS while still delivering the same precision.
	We suggest to initialize the IHS with the classical sketch based on the $A$-optimal deterministic subsampling matrix.  To our best knowledge, this is the first attempt to take initialization into account for improving the IHS method.
	\item {\bf Improved preconditioner.} It is critical to find a well-designed preconditioner so that it may be fixed during the iterative sketch process. We propose to construct the preconditioner from $A$-optimal subsample and refine it by adding a ridge term.  Unlike complicated random projection-based methods, we obtain our preconditioner at a low cost by recycling the subsamples in initialization \revise{and make no assumption on the sample size.}
	\item {\bf Adaptive step lengths.} We modify the IHS to be an adaptive first-order method by using a fixed preconditioner subject to  variable step lengths.   The step lengths at each iteration are determined by the exact line search, which ensures the algorithm to enjoy the guaranteed convergence.  
\end{itemize}

Through extensive experiments, our proposed method is shown to achieve the state-of-art performance in terms of both precision and speed when approximating both the ground truth $\bm\beta^*$ and the full data estimator $\bm{\hat\beta}^{\rm LS}$. 
 
\section{Reformulation of IHS}
The original IHS algorithm is displayed in Algorithm \ref{algo-IHS}. For simplicity, we omit the superscript from $\hat{\bm{\beta}}^{\rm IHS}$.  As for the sketch matrix, we use SRHT as it is widely adopted in the literature. 
Following \cite{tropp2011improved}   and \cite{lu2013faster}, when $n=2^k$ where $k$ is a positive integer, the SRHT sketch matrix  is given by
\begin{eqnarray*}
\bm{S}=\sqrt{\frac{n}{m}}\bm{RHD},
\end{eqnarray*}
where $\bm{R}$ is an $m\times n$ matrix with rows chosen uniformly without replacement from the standard bases of $\mathbb{R}^{n}$, $\bm{H}$ is a normalized Walsh--Hadamard matrix of size  $n \times n$, and $\bm{D}$ is an $n \times n$ diagonal matrix with i.i.d. Rademacher entries. 
\revise{Note that when $n$ is not the power of two, $\bm{X}$ is padded with zeros until $n$ achieves the next greater power of two.} 

\begin{algorithm}[ht]\label{algo-IHS}
	\caption{Iterative Hessian Sketch (IHS)}
	\LinesNumbered
	\KwIn{Data $(\bm{X},\bm{y})$,  sketching dimension $m$, iteration number $N$.}
	\textbf{Initialization:} $\hat{\bm{\beta}}_0=\bm{0}$.\\ 
	\For{$t=1, \dots, N$}{
		Generate $\bm{S}_{t}\in\mathbb{R}^{m\times n}$ independently.
		\\
		$
		\Delta\hat{\bm{\beta}}_{t}=(\bm{X}^T\bm{S}_{t}^T\bm{S}_{t}\bm{X})^{-1}\bm{X}^T(\bm{y}-\bm{X}\hat{\bm{\beta}}_{t-1})$.\\
		$\hat{\bm{\beta}}_{t}=\hat{\bm{\beta}}_{t-1}+\Delta\hat{\bm{\beta}}_{t}$.
	}
	\KwResult{$\hat{\bm{\beta}}=\hat{\bm{\beta}}_N$}
\end{algorithm}

In what follows, we present a novel reformulation of IHS. At each step, the update formula in Algorithm \ref{algo-IHS} can be rewritten as
\begin{eqnarray}
\hat{\bm{\beta}}_t&=&\hat{\bm{\beta}}_{t-1}+(\underbrace{\bm{X}^T\bm{S}^T_t\bm{S}_t\bm{X}}_{\bm{M_t}})^{-1}\bm{X}^T(\bm{y}-\bm{X}\hat{\bm{\beta}}_{t-1})\label{IHS-rf0}\\
&=&\hat{\bm{\beta}}_{t-1}+\bm{M}_t^{-1}\bm{X}^T\bm{y}-\underbrace{\bm{M}_t^{-1}\bm{X}^T\bm{X}}_{\bm{A}_t}\hat{\bm{\beta}}_{t-1}\nonumber\\
&=&(\bm{I}_d-\bm{A}_t)\hat{\bm{\beta}}_{t-1}+\bm{A}_t\hat{\bm{\beta}}^{\rm LS}.\label{IHS-rf1}
\end{eqnarray}
By (\ref{IHS-rf1}), we can derive the following lemma by mathematical induction.
\begin{lemma}\label{lemma1}
	Given an initializer $\hat{\bm{\beta}}_0$ and a series of independent sketch matrices $\{\bm{S}_i\}_{i=1}^t$, for any positive integer $t$, we have
	\begin{eqnarray}\label{lemma-fml}
	\hat{\bm{\beta}}_t=\prod_{i=1}^{t}(\bm{I}_d-\bm{A}_i)\hat{\bm{\beta}}_0+\left[\bm{I}_d-\prod_{i=1}^{t}(\bm{I}_d-\bm{A}_i)\right]\hat{\bm{\beta}}^{\rm LS}.
	\end{eqnarray}
where $\bm{A}_t = \bm{M}_t^{-1}\bm{X}^T\bm{X}$ and $ \bm{M}_t = \bm{X}^T\bm{S}^T_t\bm{S}_t\bm{X}$. 
\end{lemma}

Lemma~\ref{lemma1} reveals that $\hat{\bm{\beta}}_t$ is a linear combination of the initial $\hat{\bm{\beta}}^{\rm IHS}_0$ and the full data estimator $\hat{\bm{\beta}}^{\rm LS}_0$, with the  re-weighting matrices during the iterations. 
We can therefore study the convergence properties of  $\hat{\bm{\beta}}_t$ by looking into $\bm{A}_t$ or $\bm{M}_t$. The following theorem provides a sufficient isometric condition for the convergence guarantee. 

\begin{theorem}\label{thm1}
Given an initial estimator $\hat{\bm{\beta}}_0$ and a set of sketch matrices $\{\bm{S}_t\}_{t=1}^\infty$, the IHS estimator $\hat{\bm{\beta}}_t$ converges to $ \hat{\bm{\beta}}^{\rm LS}$ with \revise{geometric convergence},
\begin{eqnarray}\label{thm1-ieq1}
\|\hat{\bm{\beta}}_t-\hat{\bm{\beta}}^{\rm LS}\|_2\leq\left(\frac{\max\{\varepsilon_1,\varepsilon_2\}}{1-\varepsilon_1}\right)^t\|\hat{\bm{\beta}}_0-\hat{\bm{\beta}}^{\rm LS}\|_2,
\end{eqnarray}
provided that 
for any $\varepsilon_1\in (0, 1/2)$ and $\varepsilon_2 \in  (0,1-\varepsilon_1)$ and for any $\bm{a}\in \mathbb{R}^{d}$\revise{$\backslash\{\bm{0}\}$},
\begin{eqnarray}\label{isocondi}
1-\varepsilon_1 \leq \frac{\bm{a}^T\bm{M}_t\bm{a}}{\bm{a}^T\bm{X}^T\bm{X}\bm{a}}\leq 1+\varepsilon_2, \  t=1,2,\ldots
\end{eqnarray}

\end{theorem}
\begin{proof}
	From Lemma \ref{lemma1}, we have
	$$\hat{\bm{\beta}}_t-\hat{\bm{\beta}}^{\rm LS}=\prod_{i=1}^{t}(\bm{I}_d-\bm{A}_i)(\hat{\bm{\beta}}_0-\hat{\bm{\beta}}^{\rm LS}).$$
	Thus, it holds that
	\begin{eqnarray}\label{thm1-eq1}
	\|\hat{\bm{\beta}}_t-\hat{\bm{\beta}}^{\rm LS}\|_2\leq\prod_{i=1}^{t}\|\bm{I}_d-\bm{A}_i\|_2\|\hat{\bm{\beta}}_0-\hat{\bm{\beta}}^{\rm LS}\|_2.
	\end{eqnarray}
	Note that 
	\begin{eqnarray*}
	\|\bm{I}_d-\bm{A}_i\|_2&=&\|\bm{I}_d-\bm{M}_i^{-1}\bm{Q}\|_2\nonumber\\
	&\leq&\|\bm{M}_i^{-1}\bm{Q}\|_2\|\bm{I}_d-\bm{Q}^{-1}\bm{M}_i\|_2,
	\end{eqnarray*}
	where $\bm{Q}=\bm{X}^T\bm{X}$.
	Since $\bm{M}_i^{-1}\bm{Q}$ shares the same eigenvalues with $\bm{Q}^{1/2}\bm{M}_i^{-1}\bm{Q}^{1/2}$ \revise{(See 6.54 in \cite{seber2008matrix})}, we have
	$$\|\bm{M}_i^{-1}\bm{Q}\|_2=\|\bm{Q}^{1/2}\bm{M}_i^{-1}\bm{Q}^{1/2}\|_2,$$
	$$\|\bm{I}_d-\bm{Q}^{-1}\bm{M}_i\|_2=\|\bm{I}_d-\bm{Q}^{-1/2}\bm{M}_i\bm{Q}^{-1/2}\|_2.$$
	From the conditions  (\ref{isocondi}), we know that
	\begin{align}\label{isocondi2}
	1-\varepsilon_1\leq\frac{\bm{b}^T\bm{Q}^{-1/2}\bm{M}_i\bm{Q}^{-1/2}\bm{b}}{\bm{b}^T\bm{b}}\leq 1+\varepsilon_2,
	\end{align}
	for any $\bm{b}=\bm{Q}^{1/2}\bm{a}\in\mathbb{R}^d$\revise{$\backslash\{\bm{0}\}$}.
	Thus, we have that
	\begin{eqnarray*}
	\|\bm{M}_i^{-1}\bm{Q}\|_2 &\leq& \frac{1}{1-\varepsilon_1},\\
	\|\bm{I}_d-\bm{Q}^{-1}\bm{M}_i\|_2 &\leq&  \max\{\varepsilon_1,\varepsilon_2\},\\
	\|\bm{I}_d-\bm{A}_i\|_2&\leq & \frac{ \max\{\varepsilon_1,\varepsilon_2\}}{1-\varepsilon_1}.
	\end{eqnarray*}
	So, (\ref{thm1-eq1}) becomes
	\begin{eqnarray*}
	\|\hat{\bm{\beta}}_t-\hat{\bm{\beta}}^{\rm LS}\|_2\leq\left(\frac{ \max\{\varepsilon_1,\varepsilon_2\}}{1-\varepsilon_1}\right)^t\|\hat{\bm{\beta}}_0-\hat{\bm{\beta}}^{\rm LS}\|_2.
	\end{eqnarray*}
	It is clear that the rate $\in(0,1)$, so $\hat{\bm{\beta}}_t$ converges to $\hat{\bm{\beta}}^{\rm LS}$.\qed
\end{proof}

Theorem~\ref{thm1} is meaningful in different ways. \revise{Firstly,} when $\varepsilon_1=\rho, \varepsilon_2=\rho/2$ and $\rho\in (0,1/2)$, it can be checked that Theorem~\ref{thm1} corresponds to the main result of \cite{pilanci2016iterative} under the so-called \emph{good event} condition.  
 It also mimics the Johnson-Lindenstrauss lemma \citep{johnson1984extensions} when $\varepsilon_1=\varepsilon_2$.
\revise{Moreover, our theorem makes no assumption about the randomness of $\bm{S}_t$ and it is applicable to all kinds of sketch matrices satisfying (\ref{isocondi}). In the case of random $\bm{S}_t$, let $\varepsilon_1=\varepsilon_2=\varepsilon$, $\bm{U}=\bm{X}\bm{Q}^{-1/2}$ be the orthonormal basis of $\bm{X}$'s column space. It suffices to obtain (\ref{isocondi2}) and hence (\ref{isocondi}) once $\bm{S}_t$ satisfies $\|\bm{I}_d-\bm{U}^T\bm{S}_t^T\bm{S}_t\bm{U}\|_2\leq \varepsilon$. Theorem 2.4 in \cite{woodruff2014sketching} showed one way to construct such random $\bm{S}_t$ with high probability.} 


The original IHS method requires calculating $\bm{M}_t$ with time complexity $O(nd\log(d))$ repeatedly for each $t$; \revise{see Section 2.5 of \cite{pilanci2016iterative}.} Theorem~\ref{thm1} indicates that a fixed sketch matrix such $\bm{S}_t=\bm{S}_1, \forall t\geq2$ can also ensure the convergence if the condition (\ref{isocondi}) is satisfied.  This result provides enables us to reduce the computational cost for the original IHS method. For example, \cite{wang2018large} proposed the pwGradient method to improve the IHS by constructing a well-designed sketch algorithm, which can be viewed as an application of Theorem~\ref{thm1} by letting $\varepsilon_1=2\theta-\theta^2$, $\varepsilon_2=2\theta+\theta^2$ and $\theta\in (0,1/4)$. 

In the meanwhile, we can also interpret the IHS approach based on a transformed space. Let the preconditioner $\bm{M}_t=\bm{M}$, multiply
 $\bm{M}^{1/2}$ to the both sides of (\ref{IHS-rf0}), and denote $\bm{B}=\bm{M}^{-1/2}\bm{X}^T\bm{X}\bm{M}^{-1/2}$ and $\hat{\bm{\eta}}_t=\bm{M}^{1/2}\hat{\bm{\beta}}_t$ for $t=1,2,\ldots$. Then, we have
\begin{eqnarray}\label{trans}
\hat{\bm{\eta}}_t=\hat{\bm{\eta}}_{t-1}+\bm{M}^{-1/2}\bm{X}^T(\bm{y}-\bm{X}\bm{M}^{-1/2}\hat{\bm{\eta}}_{t-1}),
\end{eqnarray}
which corresponds to the gradient descent update when minimizing the following least squares objective 
\begin{eqnarray}
\tilde{f}(\bm{\eta})&=&\frac{1}{2}\|\bm{X}\bm{M}^{-1/2}\bm{\eta}-\bm{y}\|_2^2 \nonumber\\
&=& \frac{1}{2}\bm{\eta}^T\bm{B}\bm{\eta}-\bm{\eta}^T\bm{M}^{-1/2}\bm{X}^T\bm{y}. \nonumber
\end{eqnarray}
\revise{The one-to-one mapping between $\hat{\bm{\beta}}_t$ and $\hat{\bm{\eta}}_t$ ensures that one can show the convergence of $\hat{\bm{\beta}}_t$ and obtain its optima via equivalently analyzing $\hat{\bm{\eta}}_t$.}
Based on this observation, \cite{wang2017sketching} proposed the acc-IHS method by fixing the preconditioner and replacing the gradient descent with the conjugate counterpart in the transformed parameter space. It is worth mentioning that both the pwGradient and acc-IHS methods belong to the randomized approach based on random projections, while the SRHT sketching needs to operate on all the entries of $\bm{X}$. We hence seek to improve the IHS method in a more efficient and deterministic way.

\section{Adaptive IHS with $A$-Optimal Subsampling}
We extend the concept of sketch matrix from randomized settings to deterministic settings, by introducing  $\bm{\delta}\in\mathbb{R}^n$ to indicate where an observation is  selected, i.e., $\delta_i=1$ if sample $(\bm{x}_i,y_i)$ or $\bm{x}_i$ is included, $\delta_i=0$ otherwise. It is assumed that $\sum_{i=1}^{n}\delta_i=m$, which implies that the corresponding sketch matrix satisfies $\bm{S}^T\bm{S}=m^{-1}\mbox{diag}(\bm{\delta})$. It is our objective to find a good $\bm{\delta}$ subject to  certain optimality criterion.

\subsection{$A$-Optimal Classical Sketch}
The convergence of $\hat{\bm{\beta}}_t$ also depends on $\|\hat{\bm{\beta}}_0-\hat{\bm{\beta}}^{\rm LS}\|_2$ as shown in (\ref{thm1-ieq1}), which motivates us to find a good initializer. We achieve this goal by proposing an $A$-optimal estimator under the classical sketch scheme.
Suppose that we select a subset of $m$ observations. 
The least squares estimator based on the subdata and the corresponding covariance matrix are given by
\begin{eqnarray}
\hat{\bm{\beta}}^{\rm CS}(\bm{\delta})&=&\left( \frac{1}{m}\sum_{i=1}^{n}\delta_i\bm{x}_i\bm{x}_i^T \right)^{-1}\left(\frac{1}{m}\sum_{i=1}^{n}\delta_i\bm{x}_iy_i\right)\label{AOPT-CS},\\
\mbox{cov}(\hat{\bm{\beta}}^{\rm CS}(\bm{\delta}))&=&\sigma^2\left(\sum_{i=1}^{n}\delta_i\bm{x}_i\bm{x}_i^T\right)^{-1}. 
\end{eqnarray}
Let $\bm{M}(\bm{\delta})=\sum_{i=1}^{n}\delta_i\bm{x}_i\bm{x}_i^T$.
Following the $A$-optimality criterion in experimental design \citep{pukelsheim1993optimal}, we seek the subdata as indicated by $\bm{\delta}$ that minimizes the averaged variance of $\hat{\bm{\beta}}(\bm{\delta})$, which is proportional to the trace of $\bm{M}^{-1}(\bm{\delta})$. Formally, our goal can be formulated as the following discrete optimization problem,
\begin{eqnarray*}
\min\limits_{\bm{\delta}\in\{0,1\}^n} {\rm Tr}[\bm{M}^{-1}(\bm{\delta})],\quad\mbox{subject to} \sum_{i=1}^n\delta_i=m.
\end{eqnarray*}
However, it is NP-hard to solve it exactly, so we turn to derive a \revise{upper} bound of $\mbox{Tr}[\bm{M}^{-1}(\bm{\delta})]$, which leads to Algorithm \ref{algo-aoptcs}.
\revise{
\begin{theorem}\label{aoptcsthm}
Let $\bm{Q}=\bm{X}^T\bm{X}$, $\kappa(\cdot)$ denote the condition number, $\mathcal{S}=\{\bm{\delta}|\bm{M}(\bm{\delta})>0, \sum_{i=1}^{n}\delta_i=m\}$ be the collection of all feasible $\bm{\delta}$ solutions and $\lambda_{\min}$ correspond to the smallest eigenvalue. Assuming that $C=\inf_{\bm{\delta}\in\mathcal{S}}\lambda_{\min}(\bm{M}(\bm{\delta}))>0$, for any given $\bm{\delta}\in\mathcal{S}$,
\begin{eqnarray*}
{\rm Tr} [\bm{M}^{-1}(\bm{\delta})]\leq \frac{1}{\lambda_{\min}(\bm{Q})}\left[d+\frac{\kappa(\bm{Q})}{C}\sum_{i=1}^{n}(1-\delta_i)\|\bm{x}_i\|^2_2\right].
\end{eqnarray*} 
\end{theorem}
\begin{proof}
Let $\bm{W}={\rm diag}(\bm{\delta})$, $\bar{\bm{W}} = {\rm diag}(\bm{1}-\bm{\delta})$, $\lambda_{\max}$ be the largest eigenvalue, $\bm{X}=\bm{U}\bm{\Sigma}\bm{V}^T$ denote the SVD of $\bm{X}$ where $\bm{U}\in\mathbb{R}^{n\times d}$, $\bm{\Sigma}\in\mathbb{R}^{d\times d}$ and $\bm{V}\in\mathbb{R}^{d\times d}$. Therefore, we have 
\begin{eqnarray}
&& {\rm Tr}[\bm{M}^{-1}(\bm{\delta})] \nonumber\\
& = & {\rm Tr}\left[(\bm{X}^T\bm{W}\bm{X})^{-1}\right] \nonumber\\
& = & {\rm Tr}\left[\bm{V}\bm{\Sigma}^{-1}(\bm{U}^T\bm{W}\bm{U})^{-1}\bm{\Sigma}^{-1}\bm{V}^T\right] \nonumber\\
& \leq &\frac{1}{\lambda_{\min}(\bm{Q})}{\rm Tr}\left[(\bm{U}^T\bm{W}\bm{U})^{-1}\right] \quad \left(\lambda_{\max}^2(\bm{\Sigma}^{-1})=\frac{1}{\lambda_{\min}(\bm{Q})}, \bm{V}^T\bm{V}=\bm{I}_d\right) \nonumber\\
& = & \frac{1}{\lambda_{\min}(\bm{Q})}{\rm Tr}\left[(\bm{I}_d-\bm{U}^T\bar{\bm{W}}\bm{U})^{-1}\right] \quad (\bm{U}^T\bm{U}=\bm{I}_d). \label{thm2-eq1}
\end{eqnarray}
It is easy to check that $\|\bm{U}^T\bar{\bm{W}}\bm{U}\|_2\leq\|\bm{U}^T\bm{U}\|_2\|\bar{\bm{W}}\|_2\leq 1$ where the first inequality follows $\bm{U}^T\bar{\bm{W}}\bm{U}\leq \bm{U}^T\bm{U}\|\bar{\bm{W}}\|_2$ and 10.47 in \cite{seber2008matrix}. We can further conclude $\|\bm{U}^T\bar{\bm{W}}\bm{U}\|_2<1$ as $\bm{\delta}\in\mathcal{S}$. According to the extension of Corollary 5.6.16. in \cite{horn2012matrix}, we have 
$$(\bm{I}_d-\bm{U}^T\bar{\bm{W}}\bm{U})^{-1}=\bm{I}_d+\sum_{k=1}^{\infty}(\bm{U}^T\bar{\bm{W}}\bm{U})^k.$$
So the trace in (\ref{thm2-eq1}) can be further bounded as follows,
\begin{eqnarray}
&&{\rm Tr}\left[(\bm{I}_d-\bm{U}^T\bar{\bm{W}}\bm{U})^{-1}\right] \nonumber\\
& = & d+\sum_{k=1}^{\infty}{\rm Tr}[(\bm{U}^T\bar{\bm{W}}\bm{U})^k] \nonumber\\
& \leq & d+\sum_{k=1}^{\infty}\|\bm{U}^T\bar{\bm{W}}\bm{U}\|^{k-1}_2 {\rm Tr}(\bm{U}^T\bar{\bm{W}}\bm{U}) \nonumber\\
& = & d + \frac{{\rm Tr}(\bm{U}^T\bar{\bm{W}}\bm{U})}{1-\|\bm{U}^T\bar{\bm{W}}\bm{U}\|_2} \quad (\mbox{Taylor series}). \label{thm2-eq2}
\end{eqnarray}
Note that the denominator
\begin{eqnarray*}
1-\|\bm{U}^T\bar{\bm{W}}\bm{U}\|_2&=&\lambda_{\min}(\bm{U}^T\bm{W}\bm{U})=\lambda_{\min}(
\bm{V}\bm{U}^T\bm{W}\bm{U}\bm{V}^T)\\
&=&\lambda_{\min}(\bm{Q}^{-1/2}\bm{M}(\bm{\delta})\bm{Q}^{-1/2}),
\end{eqnarray*}
where the second and third equations follow the definition of eigenvalue and $\bm{X}$'s SVD respectively. With 6.76 in \cite{seber2008matrix}, the denominator is lower bounded by
\begin{eqnarray*}
\lambda_{\min}(\bm{Q}^{-1/2}\bm{M}(\bm{\delta})\bm{Q}^{-1/2})&\geq& \lambda_{\min}(\bm{M}(\bm{\delta}))\lambda^2_{\min}(\bm{Q}^{-1/2})\\
&=&\lambda_{\min}(\bm{M}(\bm{\delta}))/\lambda_{\max}(\bm{Q})\\
&\geq&C/\lambda_{\max}(\bm{Q}).
\end{eqnarray*} 
Similarly, for the nominator, we have 
\begin{eqnarray*}
{\rm Tr}(\bm{U}^T\bar{\bm{W}}\bm{U})&=&{\rm Tr}(\bm{V}\bm{U}^T\bar{\bm{W}}\bm{U}\bm{V}^T)={\rm Tr}(\bm{Q}^{-1/2}\bm{M}(\bm{1}-\bm{\delta})\bm{Q}^{-1/2})\\
&\leq&\lambda^2_{\max}(\bm{Q}^{-1/2}){\rm Tr}(\bm{M}(\bm{1}-\bm{\delta}))\\
&=&\sum_{i=1}^n(1-\delta_i)\|\bm{x}_i\|^2_2/\lambda_{\min}(\bm{Q}).
\end{eqnarray*}
The inequality follows by plugging above two bounds into (\ref{thm2-eq2}).
\qed
\end{proof}
}
\revise{By Theorem \ref{aoptcsthm}, we can seek an approximately $A$-optimal design by shrinking its upper bound, i.e., maximizing $\sum_{i=1}^n\delta_i\|\bm{x}_i\|_2^2$. }
Our subsampling approach is described as follows.
\begin{algorithm}\label{algo-aoptcs}
	\caption{$A$-Optimal Classical Sketch}
	\LinesNumbered
	\KwIn{Data matrix $\bm{X}$, subsample size $m$.}
	Compute $\|\cdot\|_2$ for each sample $\bm{x_i}$.\\
	Select $m$ subsamples with the largest $\ell_2$ norm indicated by $\bm{\delta}$.\\
	Obtain the least square estimator $\hat{\bm{\beta}}^{\rm CS}(\bm{\delta})$ on the subset following (\ref{AOPT-CS}).
\end{algorithm}
\begin{remark}
	The time for norm calculation is $O(nd)$, while sorting requires on average $O(n)$ operations \revise{\citep{martinez2004partial}}. For step 3, it costs $O(md^2+d^3)$ to obtain $\hat{\bm{\beta}}^{\rm CS}(\bm{\delta})$. In total, the time complexity of Algorithm \ref{algo-aoptcs} is $O(nd+md^2)$. The time can be further reduced to $O(nd)$ when it comes to $n>md$, \revise{which is a common scenario for massive data. Furthermore, this constraint can also be a reference for the sketch size determination as it implies that $m$ should not exceed $n/d$ for an efficient complexity.} From the time complexity perspective, our algorithm is as efficient as the $D$-optimality based method in \cite{wang2018information}, while ours is more suitable for parallel computation.
\end{remark}

\begin{remark}
	\revise{In practice, when the data matrix $\bm{X}$ is centered and scaled initially}, our algorithm tends to choose the extreme samples with the farthest distances to the center, which is consistent with the conclusion in \cite{wang2018information}.  
\end{remark}

As shown in Figure \ref{CS-Cpr}, the performance of the $A$-optimal method uniformly dominates that of randomized sketch matrices and classical sketching scheme. Furthermore, the MSE of our approach decreases with an increase of $n$. In this regard, our $A$-optimal estimator serves as a good initialization for further enhancements.

\begin{figure}[ht!]
	\centering
	\subfigure[Normal]{
	\begin{minipage}[b]{0.4\linewidth}
			\includegraphics[width=\linewidth]{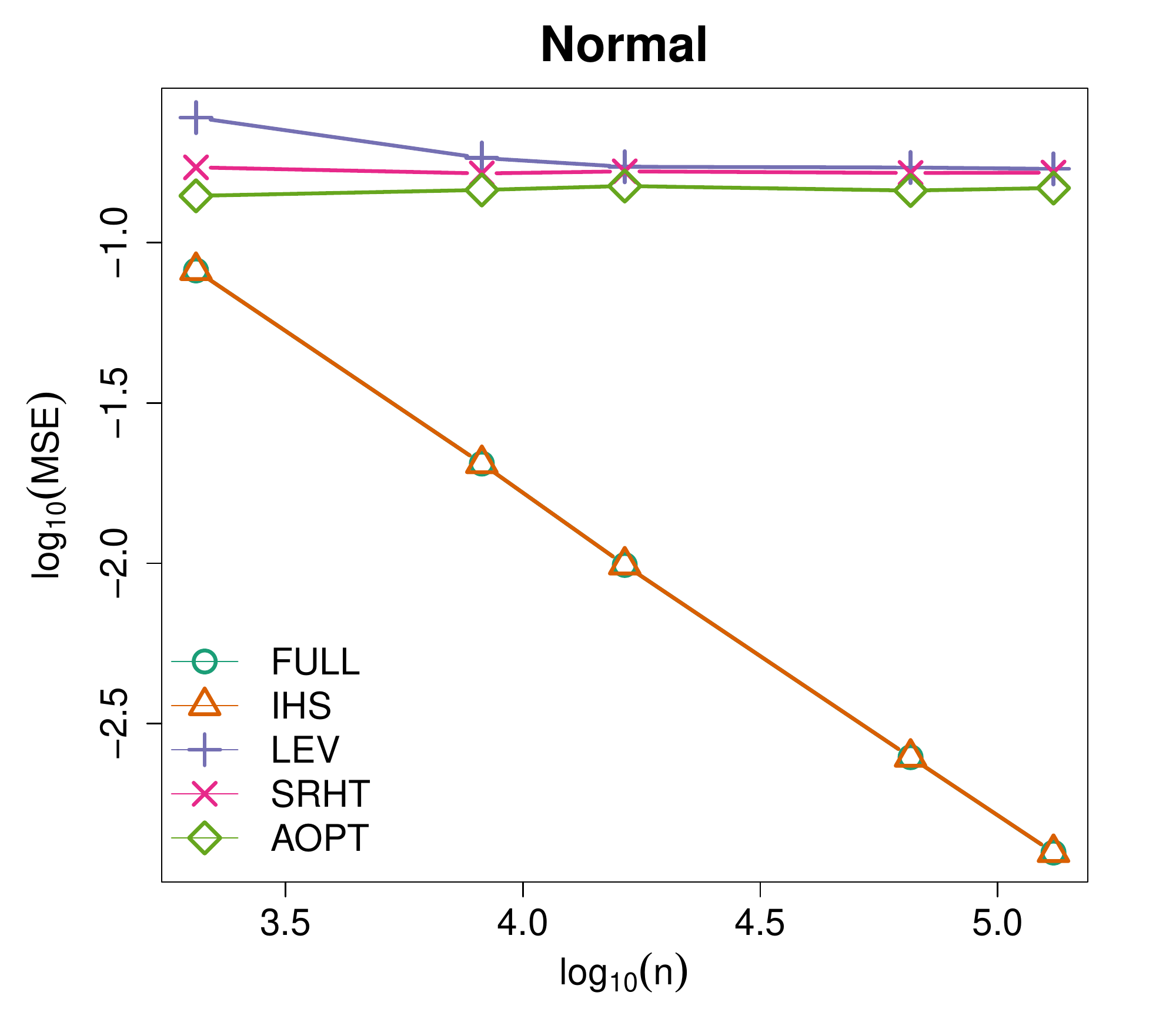}
	\end{minipage}}
	\subfigure[Log-Normal]{
	\begin{minipage}[b]{0.4\linewidth}
		\includegraphics[width=\linewidth]{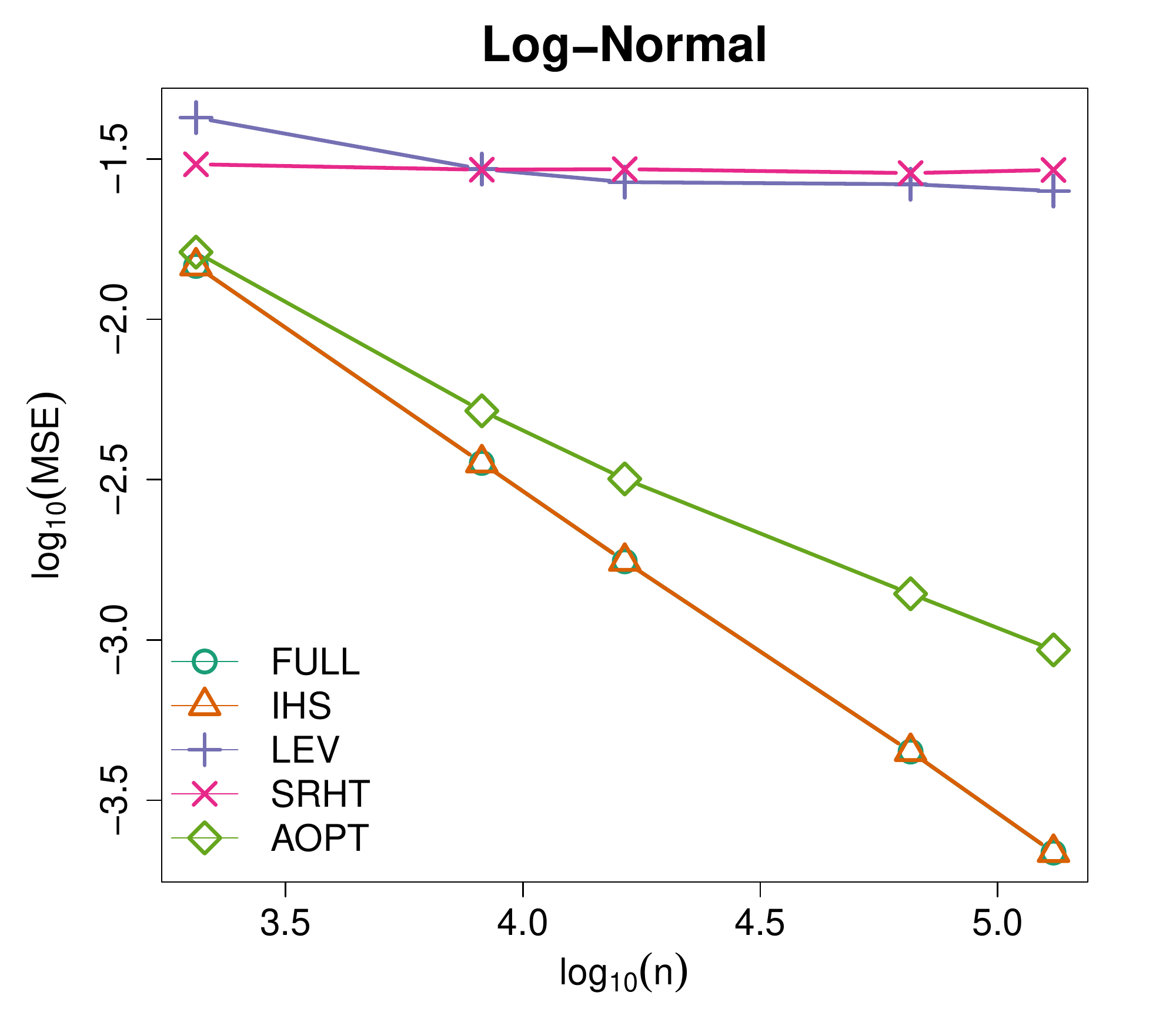}
	\end{minipage}}
	\subfigure[$t_2$]{
	\begin{minipage}[b]{0.4\linewidth}
	\includegraphics[width=\linewidth]{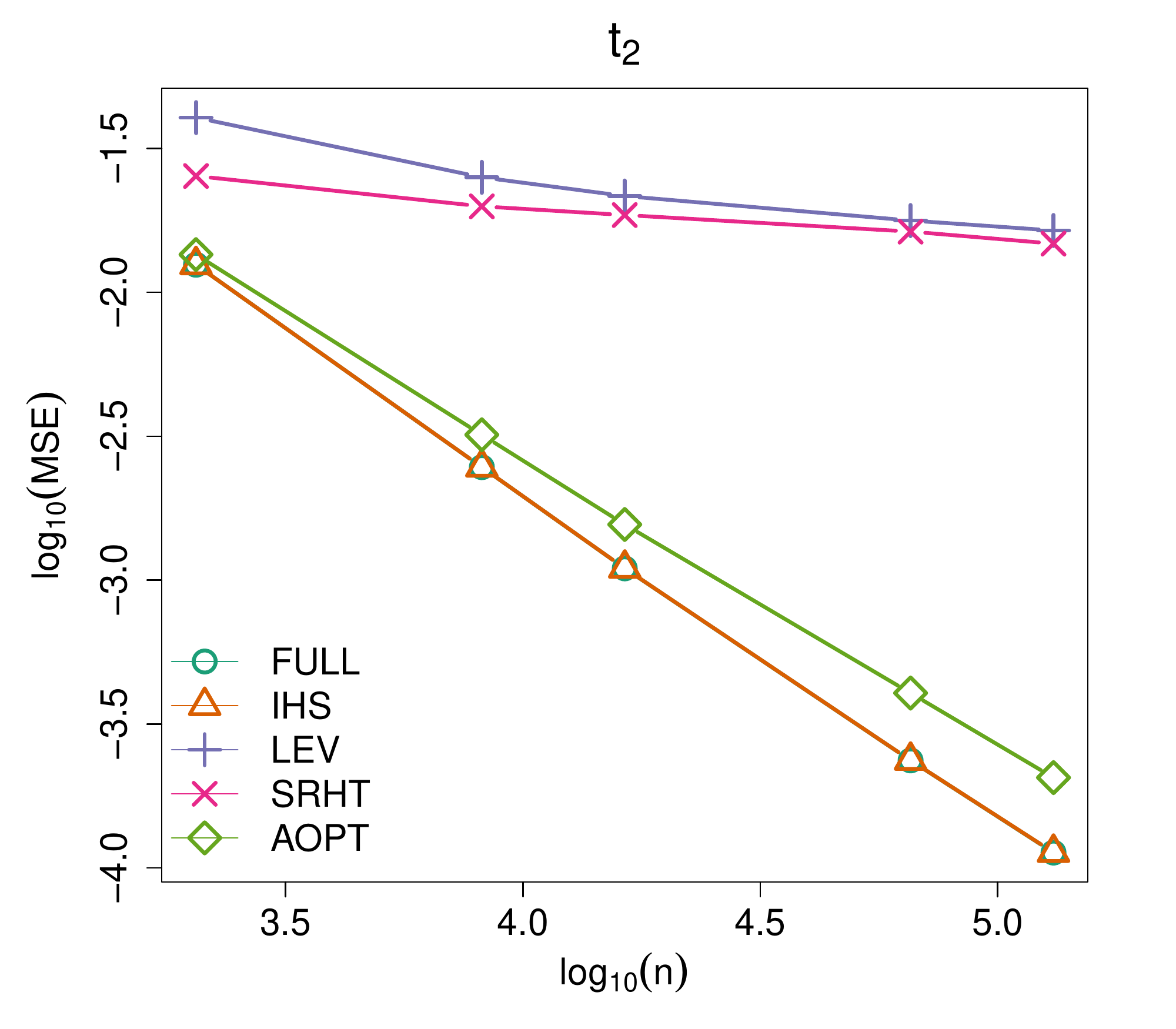}
	\end{minipage}}
	\subfigure[Mixture]{
	\begin{minipage}[b]{0.4\linewidth}
	\includegraphics[width=\linewidth]{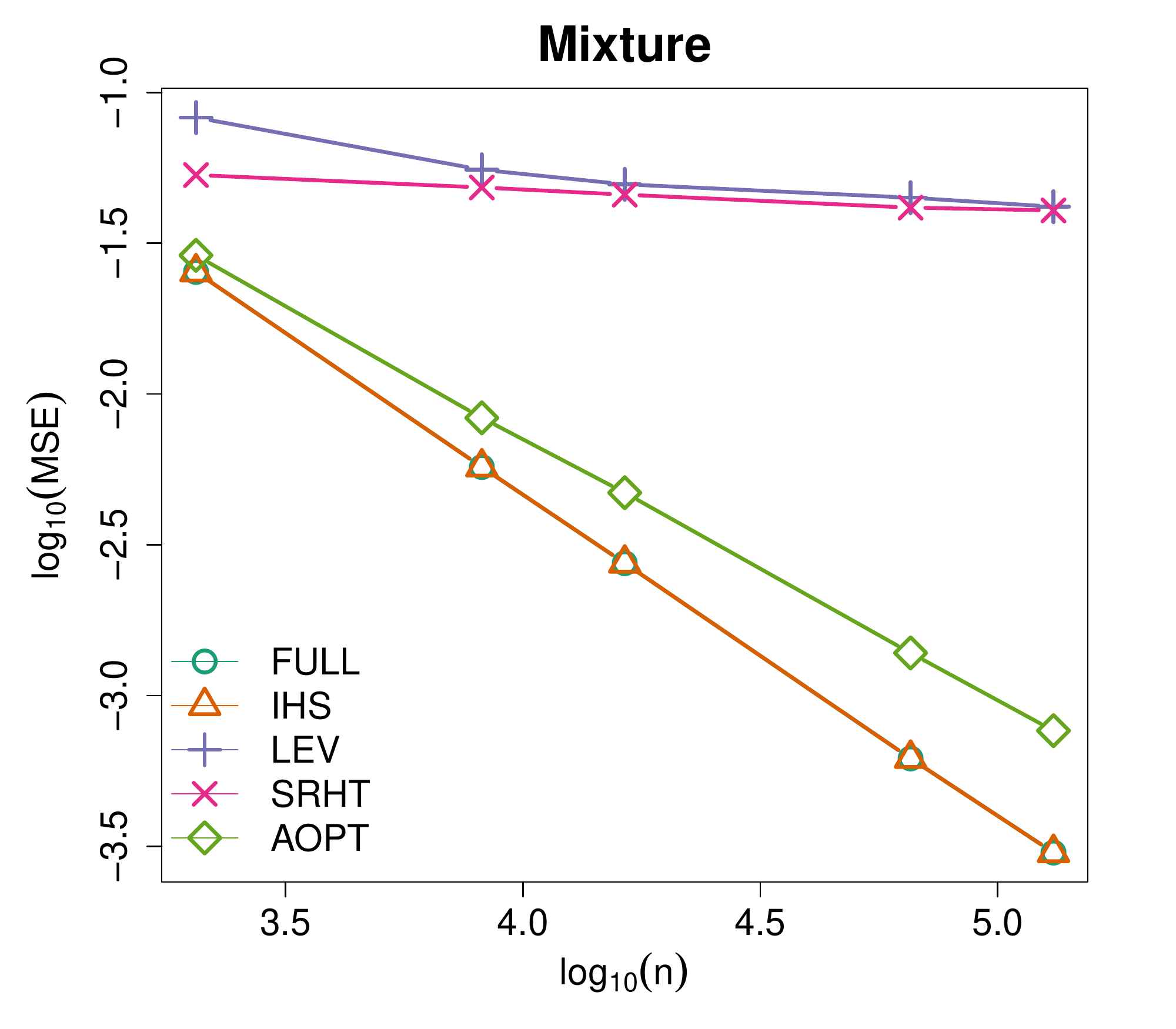}
	\end{minipage}}
	\begin{mycaption}\label{CS-Cpr}
		Plots of mean squared errors (MSE) versus the row dimensions $n\in\{2^k|k=11,13,14,16,17\}$, where $d=10$ and we conduct IHS for a total of $N=10$ rounds with a sketch size $m=10d=100$. FULL means calculating LSE with full data. SRHT, LEV and AOPT represent the estimators calculated under the classical sketch with corresponding sketch matrices. Moreover, we apply a sketch size of  $M=Nm=1000$ to compute those estimators for the sake of fairness. Each point corresponds to the result averaged over \revise{$1000$} trials.
	\end{mycaption}
\end{figure}

\subsection{Improved Preconditioner}
Note that $\hat{\bm{\beta}}$ is mainly determined by a set of matrices $\{\bm{A}_i=\bm{M}_i^{-1}\bm{X}^T\bm{X}\}_{i=1}^N$. Rather than specifying $\bm{M}_i$'s by repeatedly sketching in IHS or fixing $\bm{M}_i=\bm{M}$ in pwGradient and acc-IHS, we make a compromise by defining $\bm{M}_i=\alpha_i^{-1}\bm{M}$, that is,
\begin{equation}\label{NNN}
\hat{\bm{\beta}}_t=\prod_{i=1}^{t}(\bm{I}_d-\alpha_i\bm{A})\hat{\bm{\beta}}_0+\left[\bm{I}_d-\prod_{i=1}^{t}(\bm{I}_d-\alpha_i\bm{A})\right]\hat{\bm{\beta}}^{\rm LS},
\end{equation}
where $\bm{A}=\bm{M}^{-1}\bm{X}^T\bm{X}$.
We consider the $A$-optimal Hessian sketch in order to specify a deterministic preconditioner $\bm{M}$. Similarly, we can obtain the estimator of Hessian sketch and its covariance matrix based on an optimality criterion,
\begin{eqnarray*}
\hat{\bm{\beta}}^{\rm HS}(\bm{\delta})&=&\left(\frac{1}{m}\sum_{i=1}^{n}\delta_i\bm{x}_i\bm{x}_i^T\right)^{-1}\left(\frac{1}{n}\sum_{i=1}^{n}\bm{x}_iy_i\right)\label{AOPT-HS},\\
{\rm cov}(\hat{\bm{\beta}}^{\rm HS}(\bm{\delta}))&=&\sigma^2\bm{M}^{-1}(\bm{\delta})\bm{X}^T\bm{X}\bm{M}^{-1}(\bm{\delta}),
\end{eqnarray*}
where $\bm{M}(\bm{\delta})=n/m\sum_{i=1}^{n}\delta_i\bm{x}_i\bm{x}_i^T$.
The next theorem provides a \revise{upper} bound for the trace of the covariance matrix.
\revise{
\begin{theorem}\label{aopthsthm}
	Under the same conditions of Theorem \ref{aoptcsthm}, for any $\bm{\delta}\in\mathcal{S}$, we have 
	\begin{eqnarray*}
	{\rm Tr}[\bm{M}^{-1}(\bm{\delta})\bm{X}^T\bm{X}\bm{M}^{-1}(\bm{\delta})]
	\leq\frac{\kappa(\bm{Q})}{\lambda_{\min}(\bm{Q})}\left[d+\frac{\kappa(\bm{Q})}{C}\sum_{i=1}^{n}(1-\delta_i)\|\bm{x}_i\|^2_2\right]^2.
	\end{eqnarray*} 
\end{theorem}
\begin{proof}
	Note that 
	\begin{eqnarray*}
	{\rm Tr}[\bm{M}^{-1}(\bm{\delta})\bm{X}^T\bm{X}\bm{M}^{-1}(\bm{\delta})]
	&\leq & \lambda_{\max}(\bm{Q}){\rm Tr}[\bm{M}^{-2}(\bm{\delta})]\\
	&\leq & \lambda_{\max}(\bm{Q}){\rm Tr}[\bm{M}^{-1}(\bm{\delta})]^2.
	\end{eqnarray*}
	Therefore the above inequality can be easily obtained by applying Theorem \ref{aoptcsthm}.
	\qed
\end{proof}
}
\revise{Theorem \ref{aopthsthm} shows that the $A$-optimal design of Hessian sketched estimator can also be approximated via finding the samples with largest $\ell_2$ norm. The samples selected for initialization can be recycled for constructing the preconditioner.} We find that adding a ridge term can further improve our preconditioner
\begin{equation}\label{NNN2}
\bm{M}(\bm{\delta},\lambda)=\frac{n}{m}\sum_{i=1}^{n}\delta_i\bm{x}_i\bm{x}_i^T+\lambda\bm{I}_d.
\end{equation}
The rationale is demonstrated by Figure \ref{fig-aopths}. 
Note that the effectiveness of the precondtioner $\bm{M}$ is measured by \revise{$\kappa(\bm{M}^{-1}\bm{X}^T\bm{X})=\kappa(\bm{M}^{-1/2}\bm{X}^T\bm{X}\bm{M}^{-1/2})=\kappa(\bm{B})$}  \revise{\citep{benzi2002preconditioning}}, which should be close to 1. And $\kappa(\bm{B})$ can be visualized by the contour plot of $\tilde{f}(\bm{\eta})$ in the two-dimensional cases. Specifically, the flat degree of the contour plot indicates the size of the condition number. The larger the condition number, the more circular the contour plot. The optimal transformation is $\bm{M}=\bm{X}^T\bm{X}$ as shown by Figure \ref{fig-aopths}(b). Observing Figures 2(a) and \ref{fig-aopths}(c), $\kappa(\bm{B})$ is still close to the $\kappa(\bm{X}^T\bm{X})$, i.e., the preconditioner fails. After ridging, our preconditioner shrinks the \emph{major radius} in Figure \ref{fig-aopths}(c) and render the transformation to reach optimality as shown by Figure \ref{fig-aopths}(d).

\begin{figure}[ht!]
	\setcounter{subfigure}{0}
	\centering
	\subfigure[Original Space]{
	\begin{minipage}[b]{0.4\linewidth}
		\includegraphics[width=\linewidth]{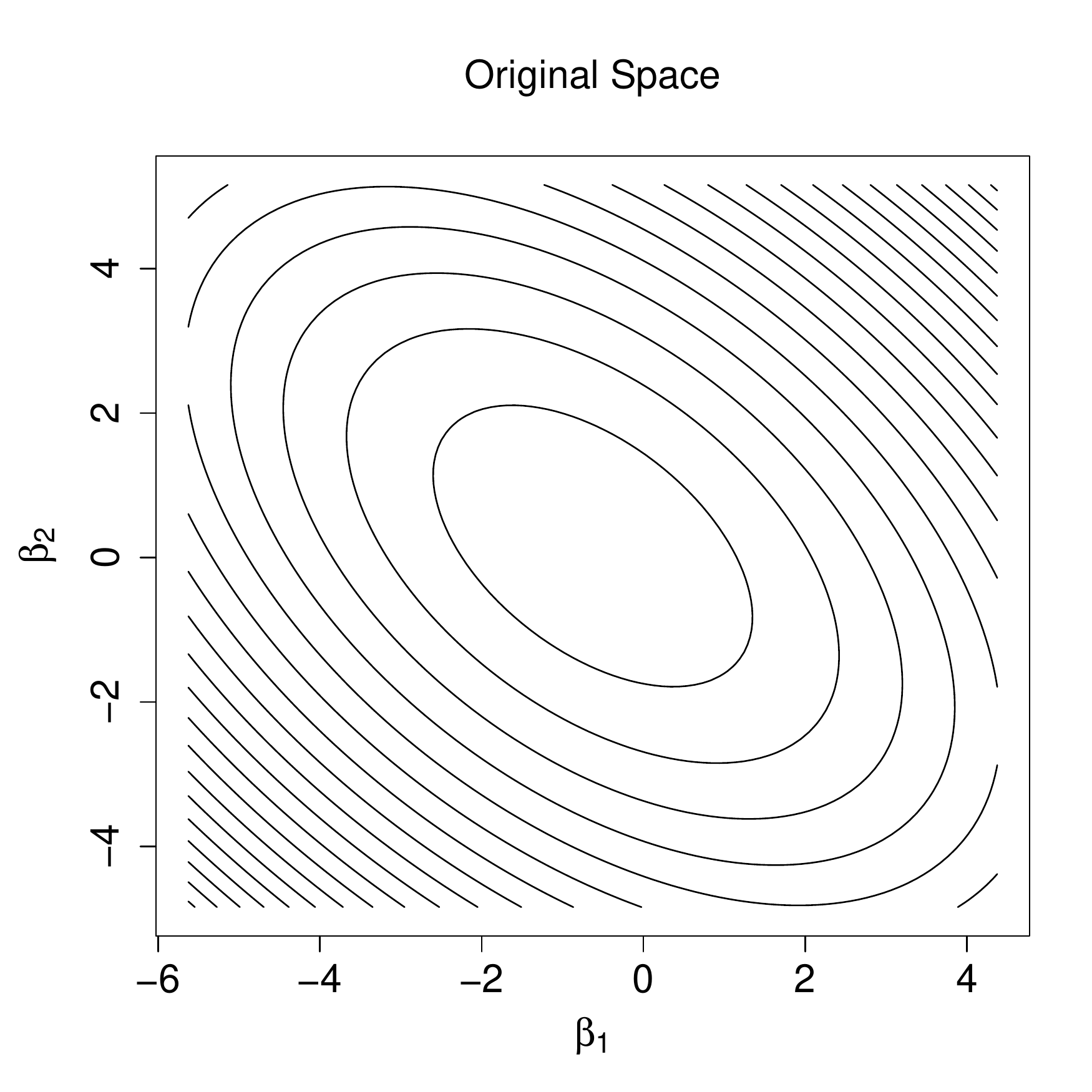}
	\end{minipage}}
	\subfigure[Optimal]{
	\begin{minipage}[b]{0.4\linewidth}
		\includegraphics[width=\linewidth]{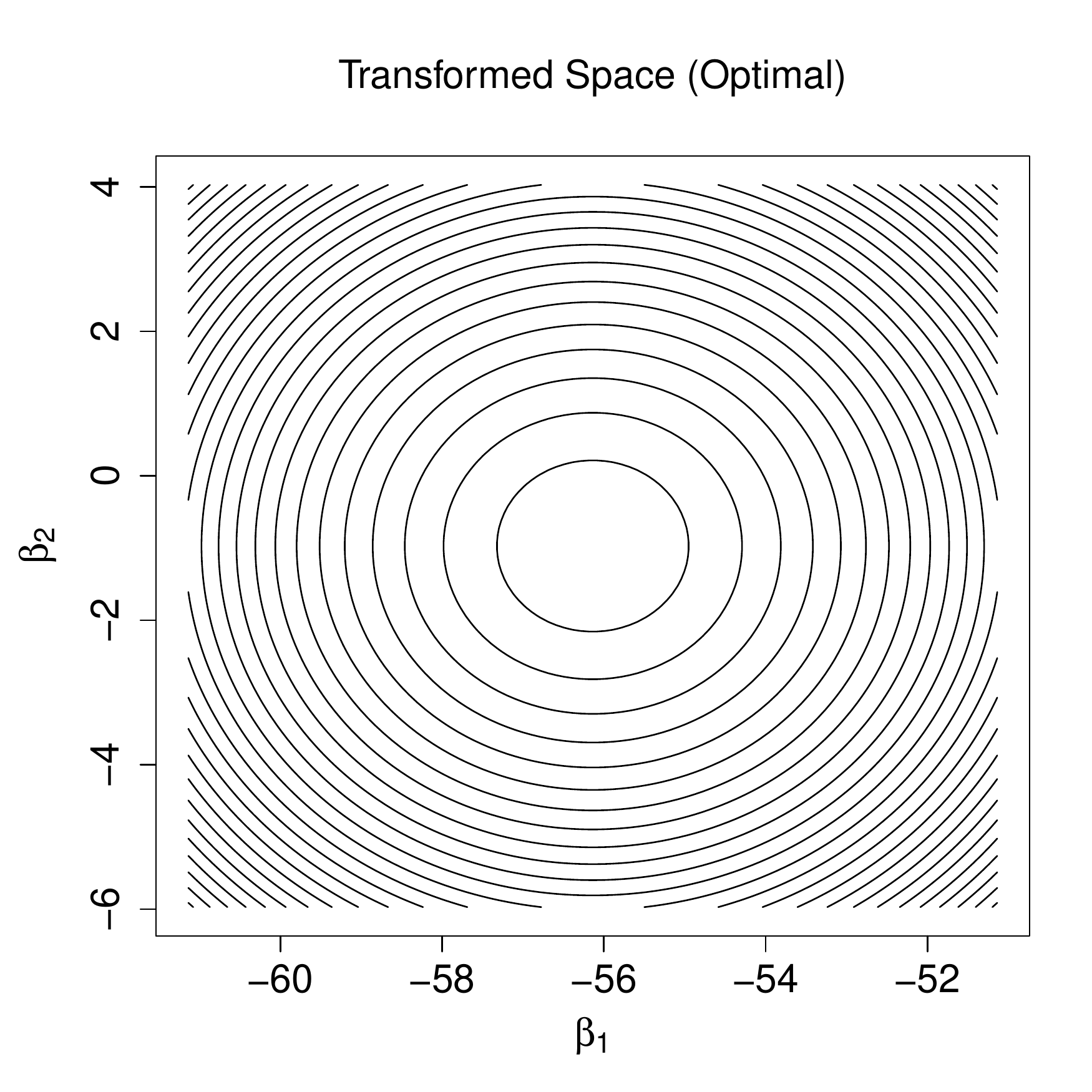}
	\end{minipage}}
	\subfigure[$\lambda=0$]{
	\begin{minipage}[b]{0.4\linewidth}
		\includegraphics[width=\linewidth]{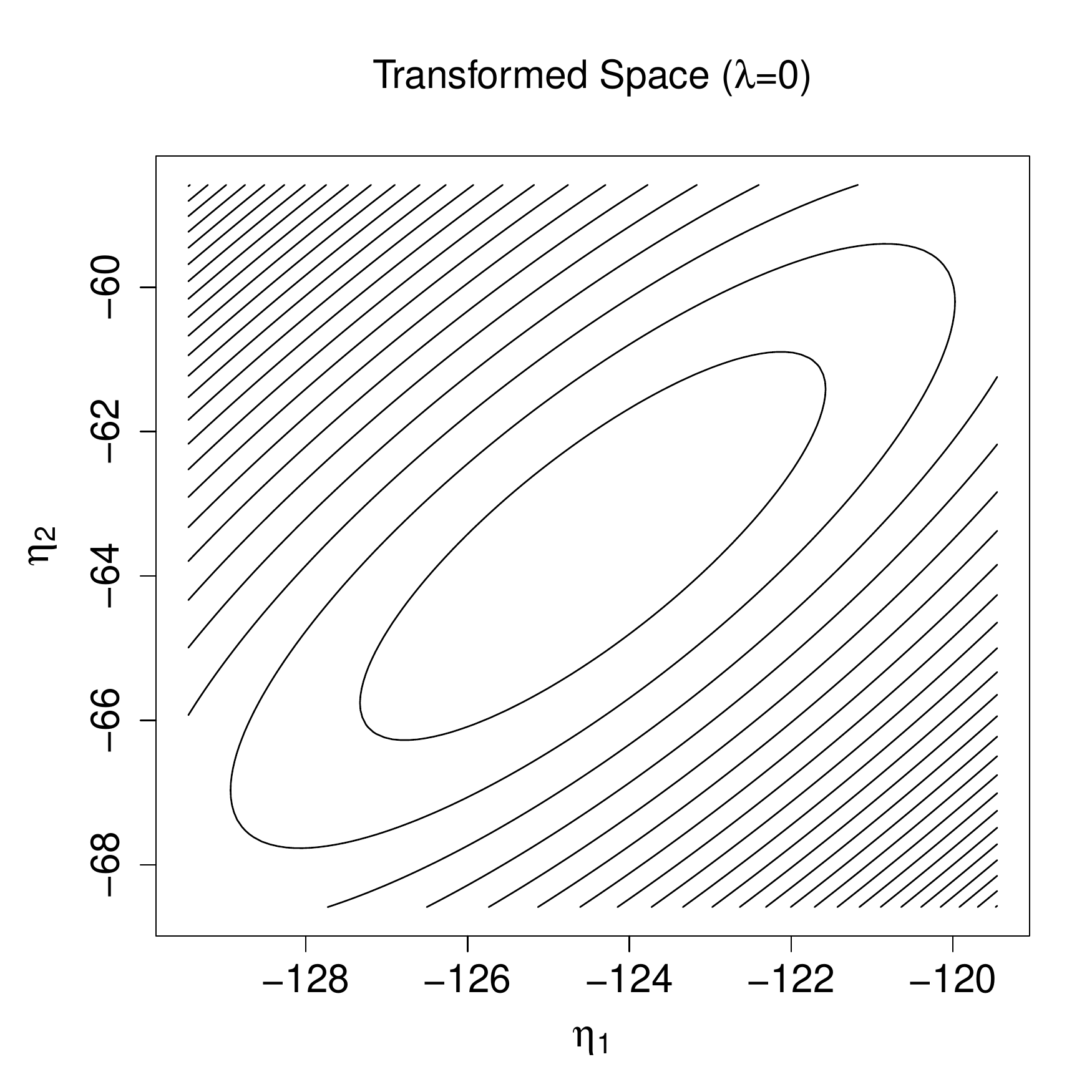}
	\end{minipage}}
	\subfigure[$\lambda=20228$]{
	\begin{minipage}[b]{0.4\linewidth}
		\includegraphics[width=\linewidth]{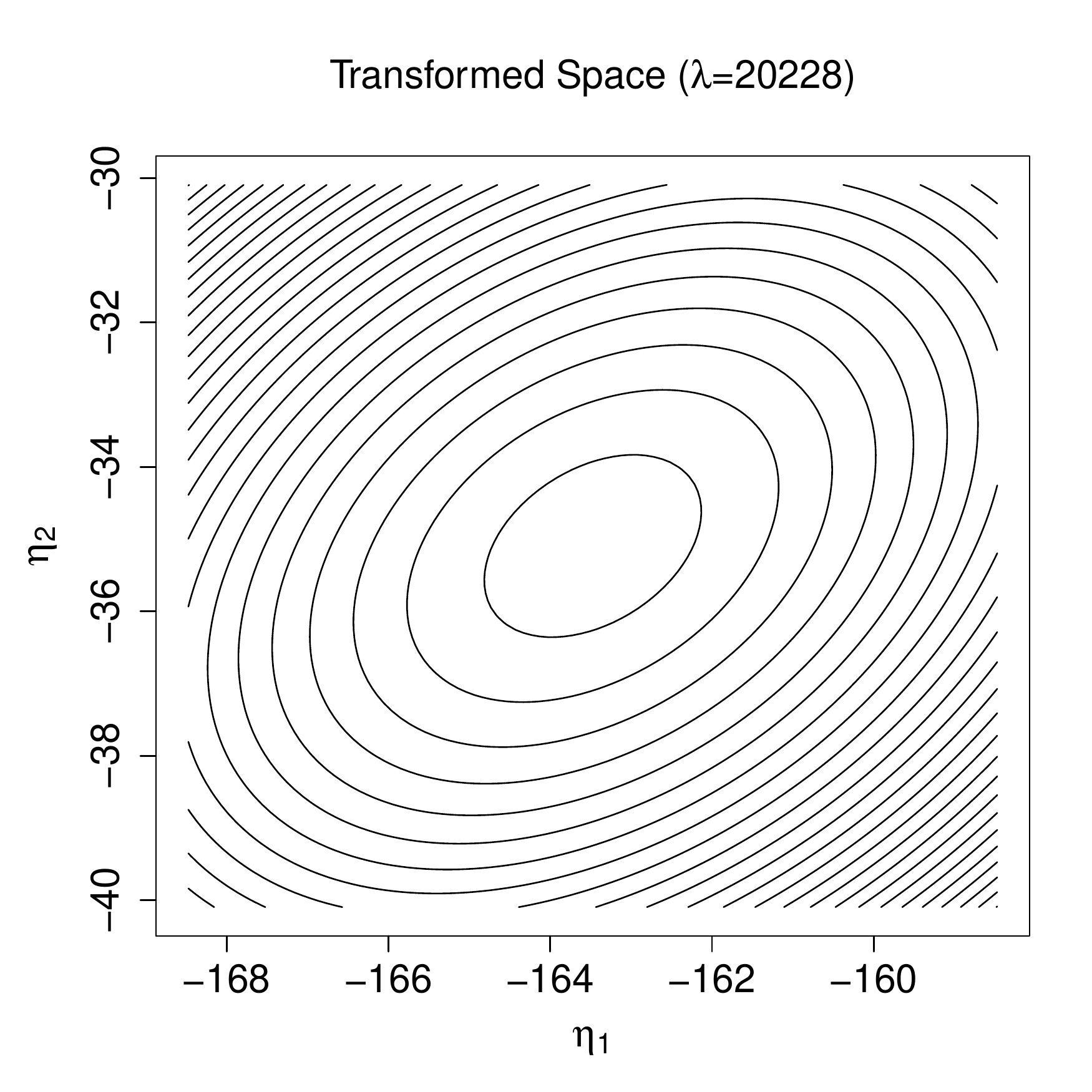}
	\end{minipage}}
	\begin{mycaption}\label{fig-aopths}
		The contour plots of $\tilde{f}(\bm{\eta})=\frac{1}{2}\|\tilde{\bm{X}}\bm{\eta}-\bm{y}\|^2_2$ where $\tilde{\bm{X}}=\bm{X}\bm{M}^{-1/2}$. (a) $\bm{M}=\bm{I}_d$, (b) $\bm{M}=\bm{X}^T\bm{X}$, (c) and (d)  $\bm{M}(\bm{\delta},\lambda)=n/m\sum_{i=1}^{n}\delta_i\bm{x}_i\bm{x}_i^T+\lambda\bm{I}_d$ with $\lambda=0$ and $\sum_{i=1}^{n}\|\bm{x}_i\|^2_2$ respectively.
	\end{mycaption}
\end{figure}

\subsection{Exact Line Search}
Using the improved preconditioner of the form (\ref{NNN2}), the adaptive first-order IHS estimator (\ref{NNN}) can be rewritten as 
\begin{equation}\label{NNN3}
\hat{\bm{\beta}}_t = \hat{\bm{\beta}}_{t-1}+\alpha_t\bm{M}^{-1}(\bm{\delta},\lambda)\bm{X}^T(\bm{y}- \bm{X} \hat{\bm{\beta}}_{t-1}),
\end{equation}
for $t=1,2,\ldots,N$.  To determine $\alpha_t$ adaptively, we note that it actually corresponds to the learning rate of the first order method in the transformed space. Recall that the learning rate in (\ref{trans}) is fixed as unit 1, which may be suboptimal and does not lead to the largest descent for every iteration. We hence fulfill the potential of IHS by exact line search. Specifically, let $$\bm{d}_t=\bm{M}^{-1/2}\bm{X}^T(\bm{y}-\bm{X}\bm{M}^{-1/2}\hat{\bm{\eta}}_{t-1})=-\nabla\tilde{f}(\hat{\bm{\eta}}_{t-1})$$ be the update direction at the $t$th iteration, the optimal step lengths are determined by adaptively minimizing the univariate function $\psi(\alpha_t)=\tilde{f}(\hat{\bm{\eta}}_{t-1}+\alpha_{t}\bm{d}_t)$. It is easy to show that
\begin{eqnarray*}
\alpha_{t}&=&\frac{\nabla\tilde{f}^T(\hat{\bm{\eta}}_{t-1})\nabla\tilde{f}(\hat{\bm{\eta}}_{t-1})}{\nabla\tilde{f}^T(\hat{\bm{\eta}}_{t-1})\bm{M}^{-1/2}\bm{X}^T\bm{X}\bm{M}^{-1/2}\nabla\tilde{f}(\hat{\bm{\eta}}_{t-1})}\nonumber\\
&=&\frac{\nabla f^T(\hat{\bm{\beta}}_{t-1})\bm{M}^{-1}\nabla f(\hat{\bm{\beta}}_{t-1})}{\nabla f^T(\hat{\bm{\beta}}_{t-1})\bm{M}^{-1}\bm{X}^T\bm{X}\bm{M}^{-1}\nabla f(\hat{\bm{\beta}}_{t-1})},
\end{eqnarray*}
where $f(\bm{\beta})=1/2\|\bm{X}\bm{\beta}-\bm{y}\|_2^2$.
Moreover, with the above $\{\alpha_{t}\}_{t=1}^N$, the first-order method becomes steepest descent in the transformed space with the guaranteed convergence; \revise{see Theorem 3.3 in \cite{Nocedal2006Numerical}}. Thus it also converges in the original space \revise{by the one-to-one mapping between $\hat{\bm{\eta}}_t$ and $\hat{\bm{\beta}}_t$.}

 In summary, we present the adaptive $A$-optimal IHS in Algorithm \ref{algo-AoptIHS}. \revise{The algorithm uses a pre-specified loop number $N$, and it can be modified by early stopping strategy, e.g. when the difference between  $\hat{\bm{\beta}}_t$ and  $\hat{\bm{\beta}}_{t-1}$ is below a tolerance threshold.}

 \begin{algorithm}[ht]\label{algo-AoptIHS}
 	\caption{$A$-Optimal IHS Algorithm}
 	\LinesNumbered
 	\KwIn{Data $(\bm{X},\bm{y})$, subsample size $m$, iteration times $N$, parameter $\lambda$ of preconditioner.}
 	Initialize $\hat{\bm{\beta}}_0$ by Algorithm \ref{algo-aoptcs} and obtain $\bm{\delta}$. \\
 	Compute $\bm{M}=\frac{n}{m}\sum_{i=1}^n\delta_i\bm{x}_i\bm{x}_i^T+\lambda\bm{I}_d$.\\
 	\For{$t=1, \dots,N$}{
 		$\bm{v}_t=\bm{X}^T(\bm{y}-\bm{X}\hat{\bm{\beta}}^{t-1})$.\\
 		$\bm{u}_t=\bm{M}^{-1}\bm{v}_t$.\\
 		$\bm{p}_t=\bm{X}
 		\bm{u}_t$.\\
 		$\alpha_t=\frac{\bm{v}_t^T\bm{u}_t}{\bm{p}_t^T\bm{p}_t}$.\\
 		$\hat{\bm{\beta}}_{t}=\hat{\bm{\beta}}_{t-1}+\alpha_t\bm{u}_t$.
 	}
 	\KwResult{$\hat{\bm{\beta}}=\hat{\bm{\beta}}_N$.}
 \end{algorithm}
 

\section{Numerical Experiments}

\subsection{Simulated Data Analysis}
\paragraph{\textbf{Data Generation.}}
We generate data by following the experimental setups of \cite{wang2018information}. All data are generated from the linear model with the true parameter $\bm{\beta}^*$ being $d$ i.i.d $N(0,1)$ variates  and $\sigma^2=9$. Let $\bm{\Sigma}$ be the covariance matrix where its $(i, j)$ entry is $\Sigma_{ij}=0.5^{I(i\neq j)}$ for $i,j \in [d]$, and $I(\cdot)$ is the indicator function. We consider four multivariate distributions for covariates $\{\bm{x}_i\}_{i=1}^n$.
\begin{enumerate}
	\item[(1)] A multivariate normal distribution $N(\bm{0},\bm{\Sigma})$.
	\item[(2)] A multivariate log-normal distribution 
	$\mbox{\rm LN}(\bm{0},\bm{\Sigma})$ \revise{which is generated by taking the exponential transformation of $N(\bm{0},\bm{\Sigma})$}.
	\item[(3)] A multivariate $t$ distribution with $2$ degrees of freedom $t_2(\bm{0},\bm{\Sigma})$.
	\item[(4)] A mixture distribution composed of five different distributions $N(\bm{1},\bm{\Sigma})$, $t_2(\bm{0},\bm{\Sigma})$, $t_3(\bm{0},\bm{\Sigma})$, $\mbox{\rm Unif}(\bm{0},\bm{2})$, $\mbox{\rm LN}(\revise{\bm{0}},\bm{\Sigma})$ with equal proportions, where $\mbox{\rm Unif}(\bm{0},\bm{2})$ represents $d$ elements from independent uniform distributions.
\end{enumerate}
To remove the effect of the intercept, we center the data. Unless particularly stated, we set $n=2^{17}, m=1000$ and present the results averaged over \revise{$R=1000$} replications.

\paragraph{\textbf{Choice of Preconditioner.}}
To assess the quality of the preconditioner $\bm{M}$ that is a positive definite matrix, we first define the $\Delta(\cdot)$ measure as the ratio of improvement with respect to the condition number,
$$
\Delta(\bm{M}) = 1-\frac{\kappa(\bm{M}^{-1/2}\bm{X}^T\bm{X}\bm{M}^{-1/2})}{\kappa(\bm{X}^T\bm{X})} = 1-\frac{\kappa(\bm{M}^{-1}\bm{X}^T\bm{X})}{\kappa(\bm{X}^T\bm{X})}.
$$
For a good preconditioner, its $\Delta$ measure should be positive and close to 1. In contrast, a preconditioner with negative $\Delta$ value would worsen the condition number $\kappa(\cdot)$. 

\revise{To choose the ridge parameter $\lambda$, we calculate the averaged $\Delta$ values on different distributions versus $\lambda$ as a proportion of $\sum_{i=1}^n\|\bm{x}_i\|^2_2$. Specifically, let $d \in \{50, 100\}$ and the proportion take values from $0.1$ to $1$ with step size $0.1$. The results are plotted in Figure \ref{fig-lambda}, which shows that the curves are similar on heavy-tailed distributions but not so consistent for the normal case. Through extensive experiments, as a rule of thumb, we suggest that $\lambda$ takes value $0.1\sum_{i=1}^{n}\|\bm{x}_i\|^2_2$ for relatively concentrated data and $0.4\sum_{i=1}^{n}\|\bm{x}_i\|^2_2$ for heavy-tailed distributions. 
\begin{figure}[ht!]
	\centering
	\setcounter{subfigure}{0}
	\subfigure[Normal]{
	\begin{minipage}[b]{0.4\linewidth}
			\centering
			\includegraphics[width=\linewidth]{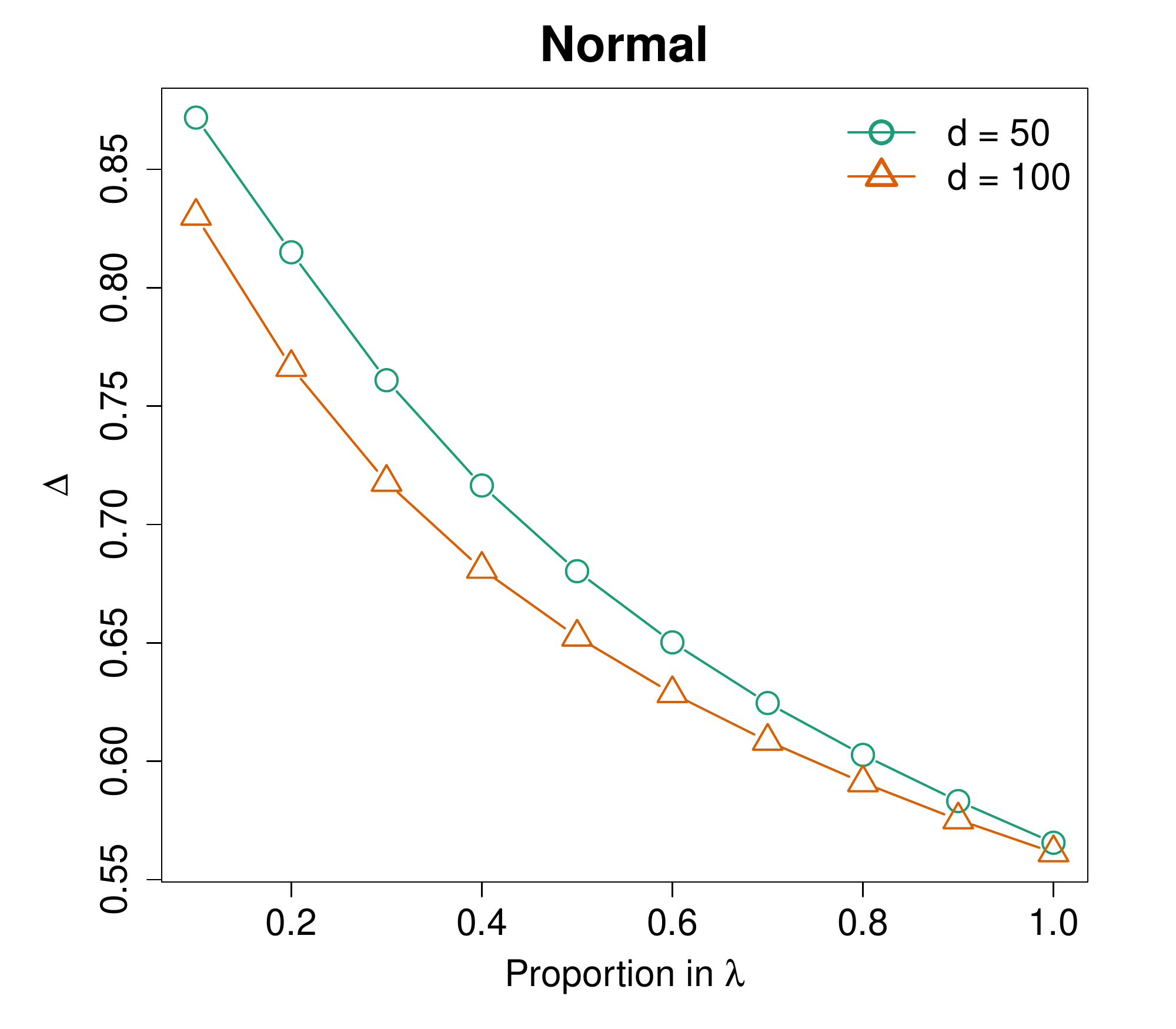}
	\end{minipage}}
	\subfigure[Log-Normal]{
	\begin{minipage}[b]{0.4\linewidth}
			\centering
			\includegraphics[width=\linewidth]{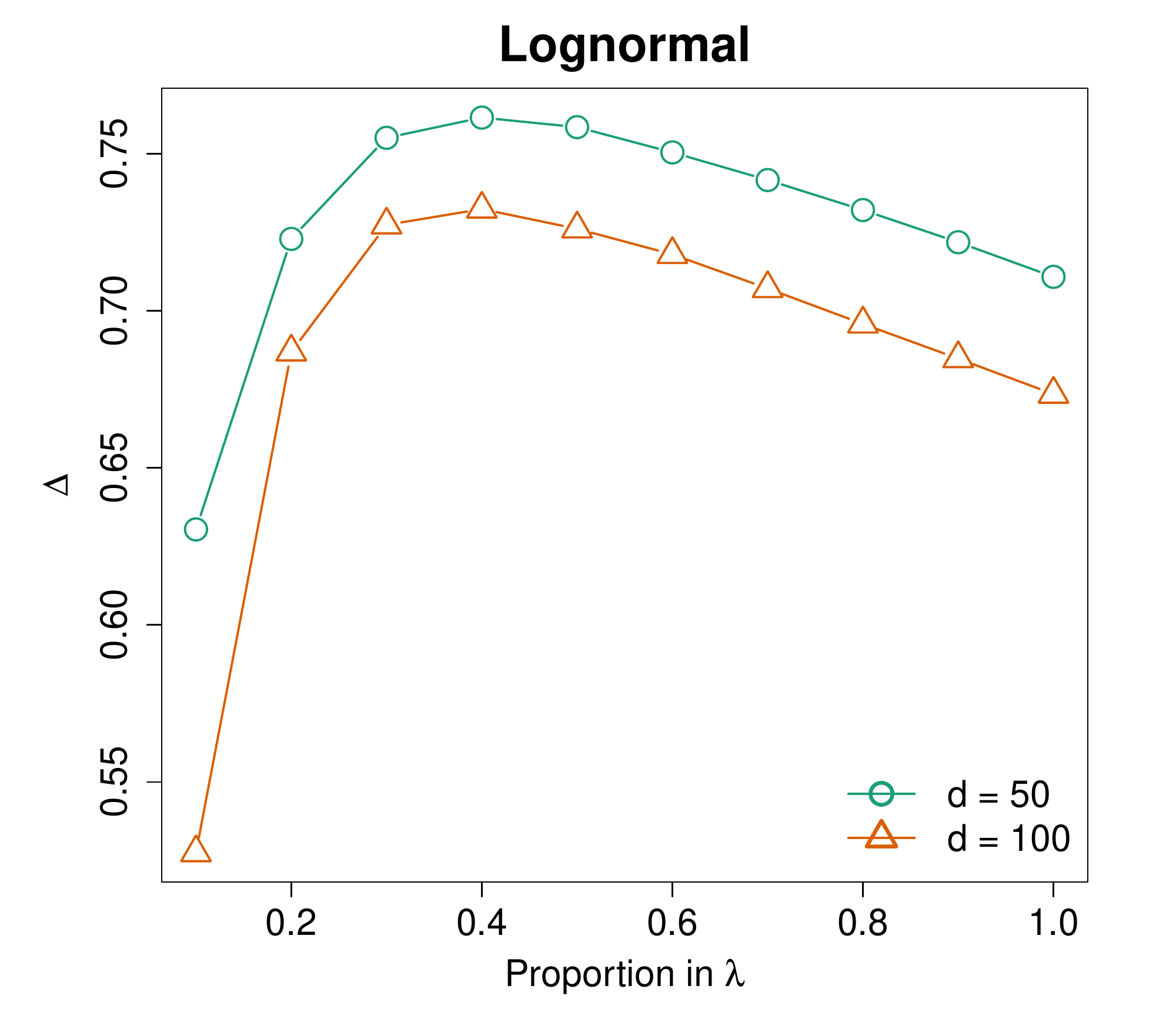}
	\end{minipage}}
	\subfigure[$t_2$]{
		\begin{minipage}[b]{0.4\linewidth}
		\centering
			\includegraphics[width=\linewidth]{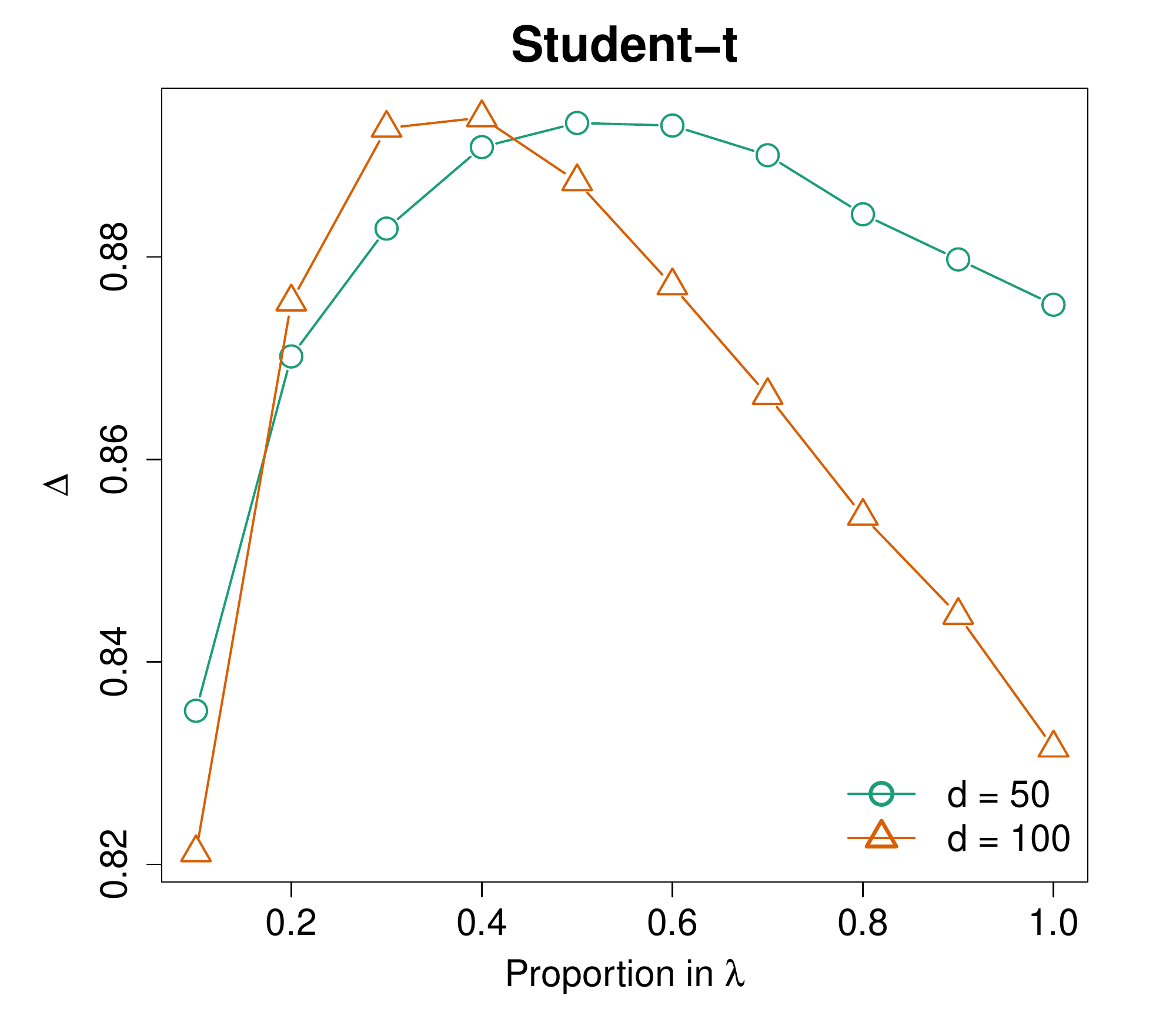}
	\end{minipage}}
	\subfigure[Mixture]{
		\begin{minipage}[b]{0.4\linewidth}
		\centering
			\includegraphics[width=\linewidth]{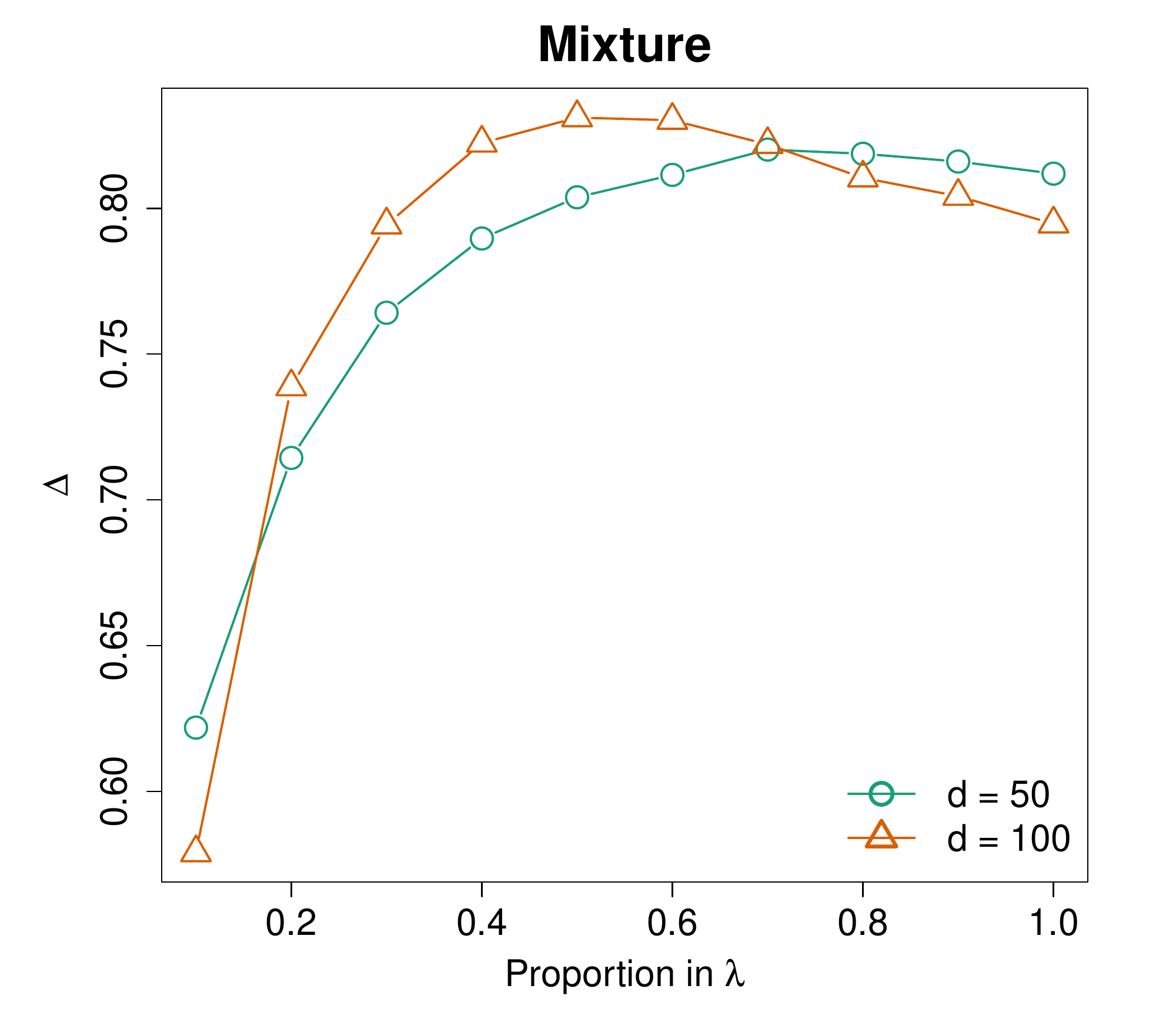}
	\end{minipage}}
	\begin{mycaption}\label{fig-lambda}
	Averaged $\Delta$-value versus $\lambda$-proportion of  $\sum_{i=1}^n\|\bm{x}_i\|^2_2$ on the ridge preconditioner $\bm{M}(\bm{\delta}, \lambda)= \frac{n}{m}\sum_{i=1}^{n}\delta_i\bm{x}_i\bm{x}_i^\top+\lambda\bm{I}_d$. The proportion varies from 0.1 to 1 with step size 0.1.
	\end{mycaption}
\end{figure}}

\revise{
Table~\ref{tb1} presents the averaged $\Delta$ values of the proposed preconditioner (\ref{NNN2}) for different distributions with dimensionality $d=50$ and  $d=100$, respectively.} We use $\tilde{\lambda}$ to denote $0.1\sum_{i=1}^{n}\|\bm{x}_i\|^2_2$ for the normal distribution and $0.4\sum_{i=1}^{n}\|\bm{x}_i\|^2_2$ for other cases. The $\Delta$ results are compared with $\lambda=0$ and the SRHT scheme. It is clear that the preconditioner with parameter $\tilde{\lambda}$ uniformly dominates in all cases. Such superiority is especially significant for the cases with $d=100$, i.e., when $d/m$ is relatively large. \revise{Furthermore, note that when $\bm{M}=\lambda\bm{I}_d$, $\kappa(\bm{M}^{-1/2}\bm{X}^T\bm{X}\bm{M}^{-1/2})=\kappa(\bm{X}^T\bm{X}/\lambda)=\kappa(\bm{X}^T\bm{X})$ and $\Delta=0$. Therefore the ridged combination indeed brings significant improvement on the preconditioner.} 

\begin{table}[ht!]
\centering
\caption{Averaged $\Delta$-values of the ridged preconditioner $\bm{M}(\lambda,\bm{\delta})=n/m\sum_{i=1}^{n}\delta_i\bm{x}_i\bm{x}_i^T+\lambda\bm{I}_d$ and SRHT. The $\tilde{\lambda}$ is chosen to be
$0.1\sum_{i=1}^{n}\|\bm{x}_i\|^2_2$ for the normal distribution and $0.4\sum_{i=1}^{n}\|\bm{x}_i\|^2_2$ for other cases. }\label{tb1}
\setlength{\tabcolsep}{15pt} 
\begin{tabular}{{l}*{5}{c}}
		\toprule
		  & Normal &
	Log-Normal &  $t_2$ &  Mixture \\
		\cmidrule(lr){2-5}
				& \multicolumn{4}{c}{$d=50$}\\

$\lambda = 0$        & -4.56  & 0.39       & 0.76  & 0.42    \\
$\lambda = \tilde{\lambda}$ & 0.87   & 0.76       & 0.89  & 0.79    \\
SRHT          & 0.55   & 0.52       & 0.75  & 0.70   \\

		& \multicolumn{4}{c}{$d=100$}\\
		
$\lambda = 0$        & -7.62  & -0.41      & 0.63  & -0.01   \\
$\lambda = \tilde{\lambda}$ & 0.83   & 0.73       & 0.90  & 0.82    \\
SRHT          & -0.04  & 0.05       & 0.46  & 0.32   \\
		\bottomrule
\end{tabular}
\end{table}

\paragraph{\textbf{Comparative Study.}}
We compare the proposed deterministic $A$-optimal IHS to the original randomized IHS \cite{pilanci2016iterative}, as well as two other improved IHS methods: acc-IHS \cite{wang2017sketching} and pwGradient \cite{wang2018large}. \reviseb{The initial estimators of these benchmark methods are set to be the vector zero as suggested in the original papers.} \revise{Two empirical MSE criteria are used to evaluate the algorithm:
\begin{eqnarray*}
{\rm MSE}_1(\hat{\bm{\beta}}_t)&=&\frac{1}{R}\sum_{i=1}^{R}\|\hat{\bm{\beta}}_{ti}-\bm{\beta}^*\|_2^2\\
{\rm MSE}_2(\hat{\bm{\beta}}_t)&=&\frac{1}{R}\sum_{i=1}^{R}\|\hat{\bm{\beta}}_{ti}-\hat{\bm{\beta}}^{\rm LS}_i\|_2^2
\end{eqnarray*}
where $\hat{\bm{\beta}}_{ti}$ denotes the estimator of the $t$th iteration in the $i$th round, and the trimmed mean with fraction 0.025 is applied to handle the effects of abnormal situations. \reviseb{We record MSEs from the $0$th iteration (i.e. initialization stage) and set the sketch dimension to be $m$ at every iteration.} Note that the randomized IHS requires independent sketch matrix for each iteration, while the other methods reuse the same sketch matrix for each iteration.}

\begin{figure}[ht!]
	\centering
	\setcounter{subfigure}{0}
	\subfigure[Normal]{
	\begin{minipage}[b]{0.4\linewidth}
			\includegraphics[width=\linewidth]{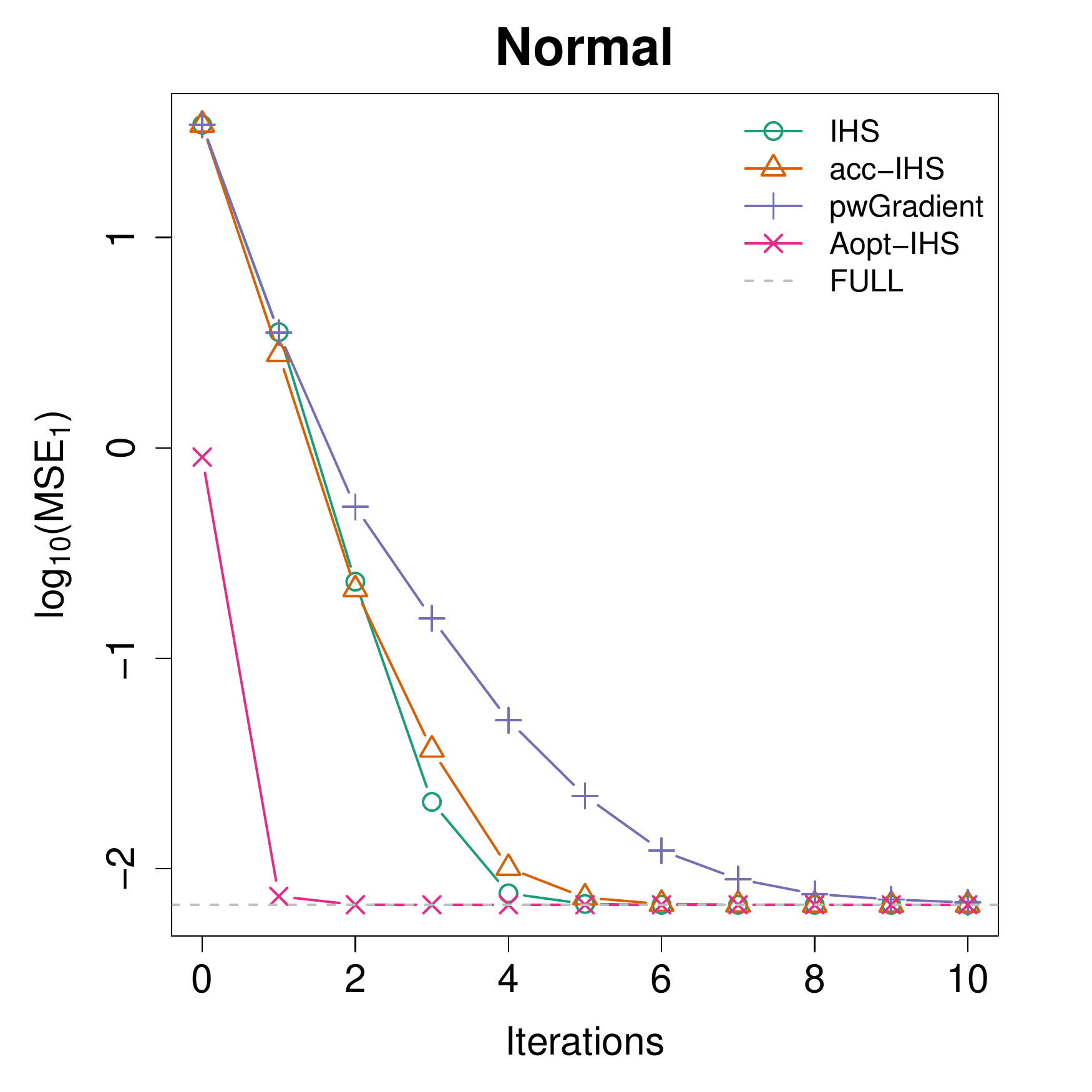}
	\end{minipage}}
	\subfigure[Log-Normal]{
		\begin{minipage}[b]{0.4\linewidth}
			\includegraphics[width=\linewidth]{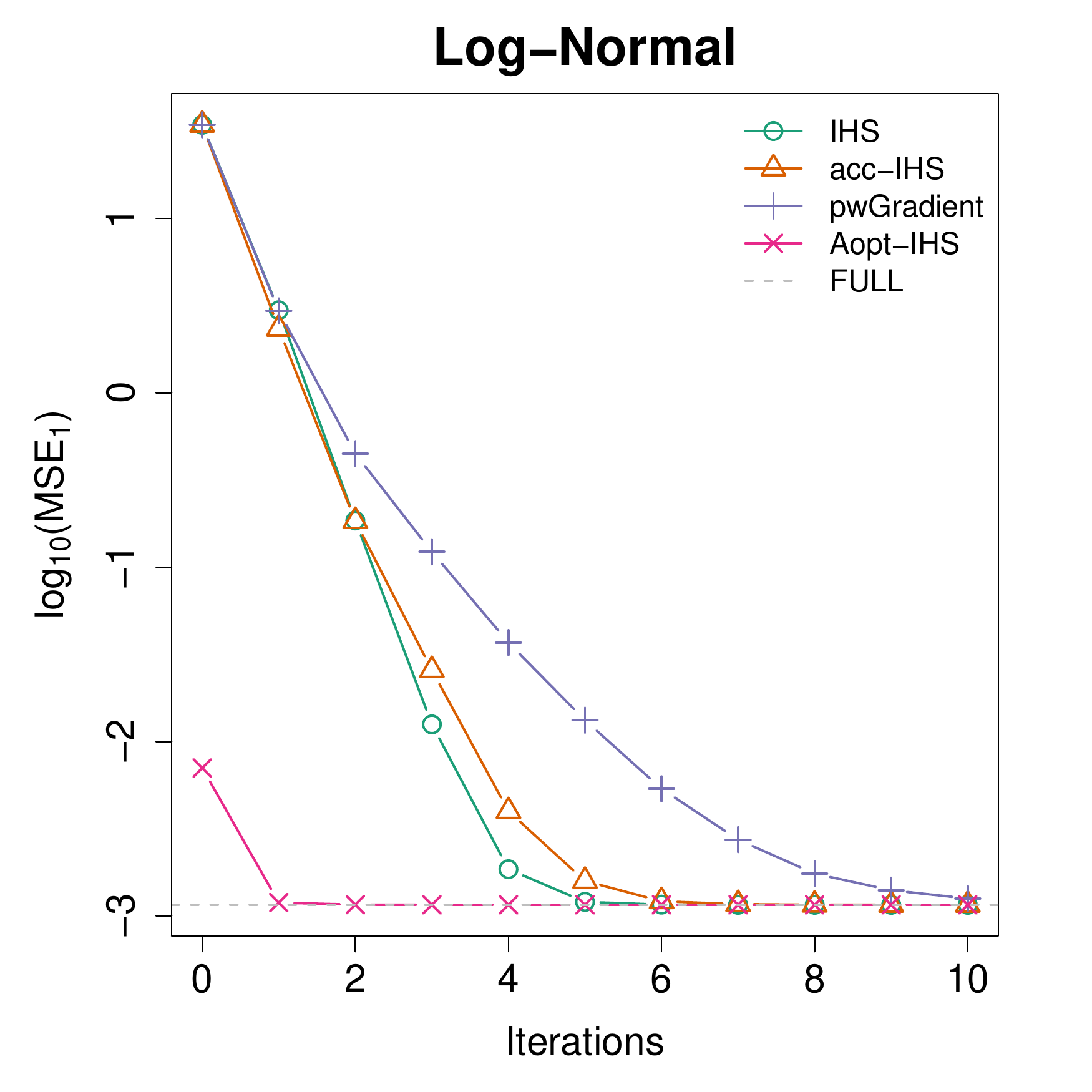}
	\end{minipage}}
	\subfigure[$t_2$]{
		\begin{minipage}[b]{0.4\linewidth}
			\includegraphics[width=\linewidth]{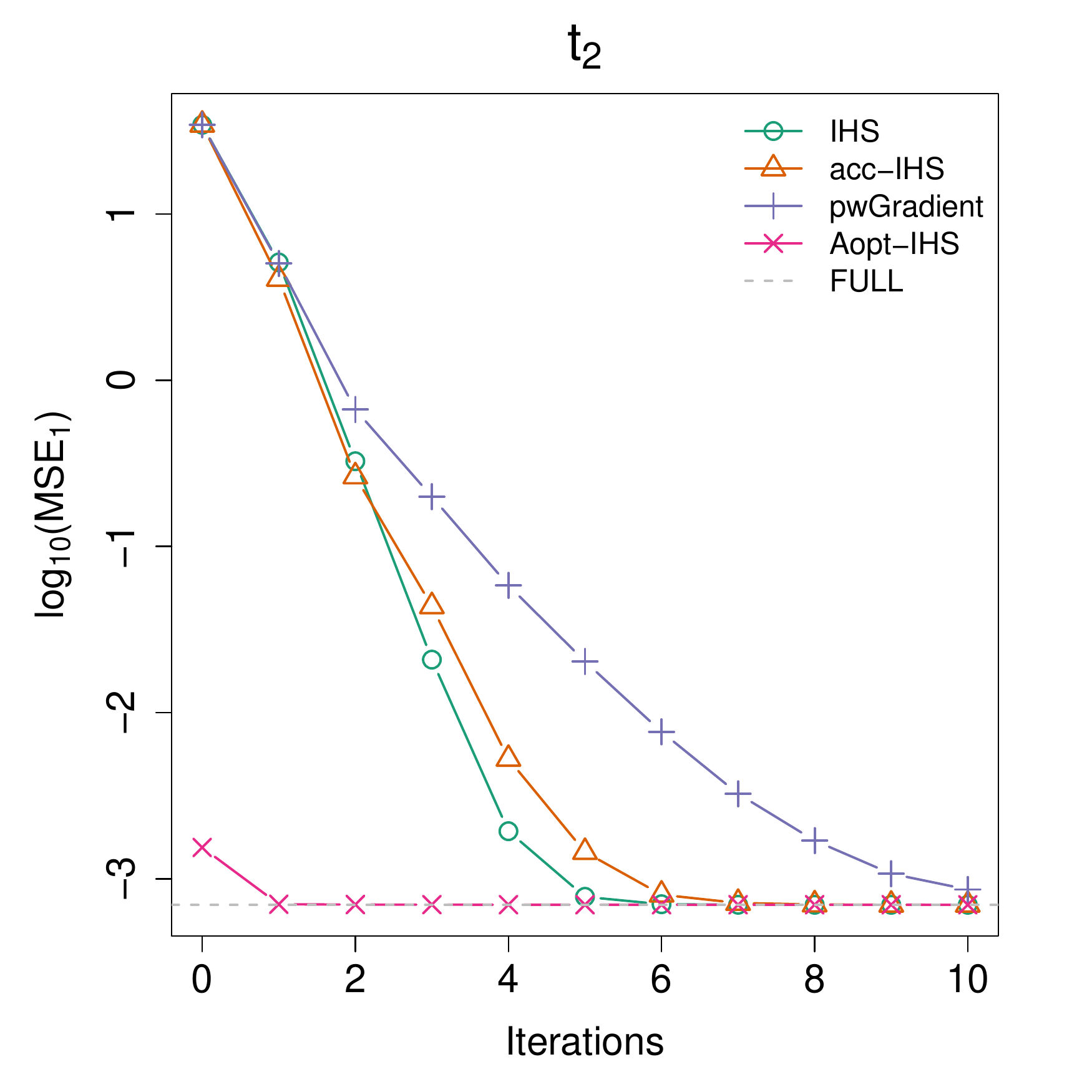}
	\end{minipage}}
	\subfigure[Mixture]{
		\begin{minipage}[b]{0.4\linewidth}
			\includegraphics[width=\linewidth]{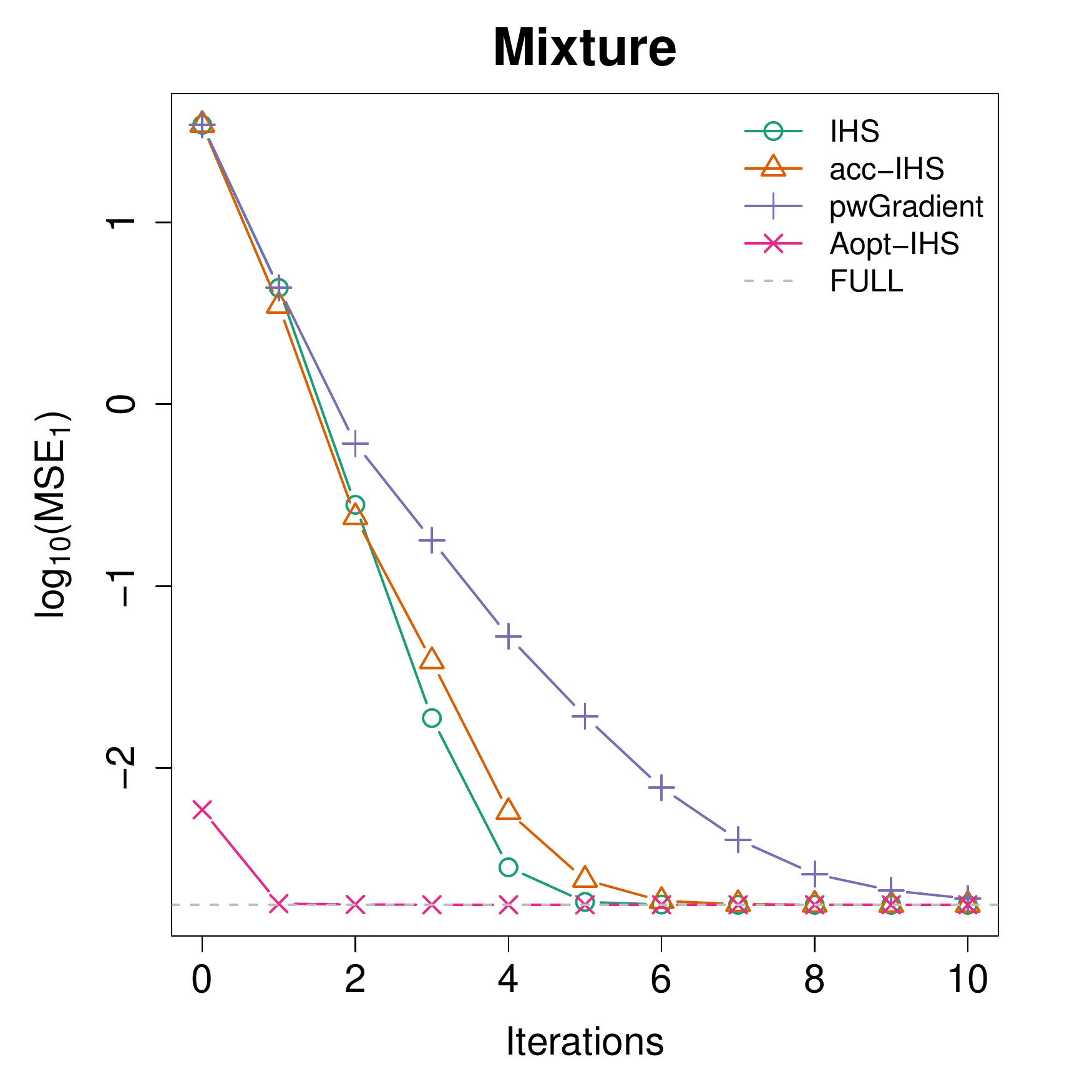}
	\end{minipage}}
	\begin{mycaption}\label{fig-true}
		Estimation error between the estimator and the ground truth versus iteration number $N$ when $d=50$.
	\end{mycaption}
\end{figure}

\begin{figure}[ht!]
	\centering
	\setcounter{subfigure}{0}
	\subfigure[Normal]{
		\begin{minipage}[b]{0.4\linewidth}
			\includegraphics[width=\linewidth]{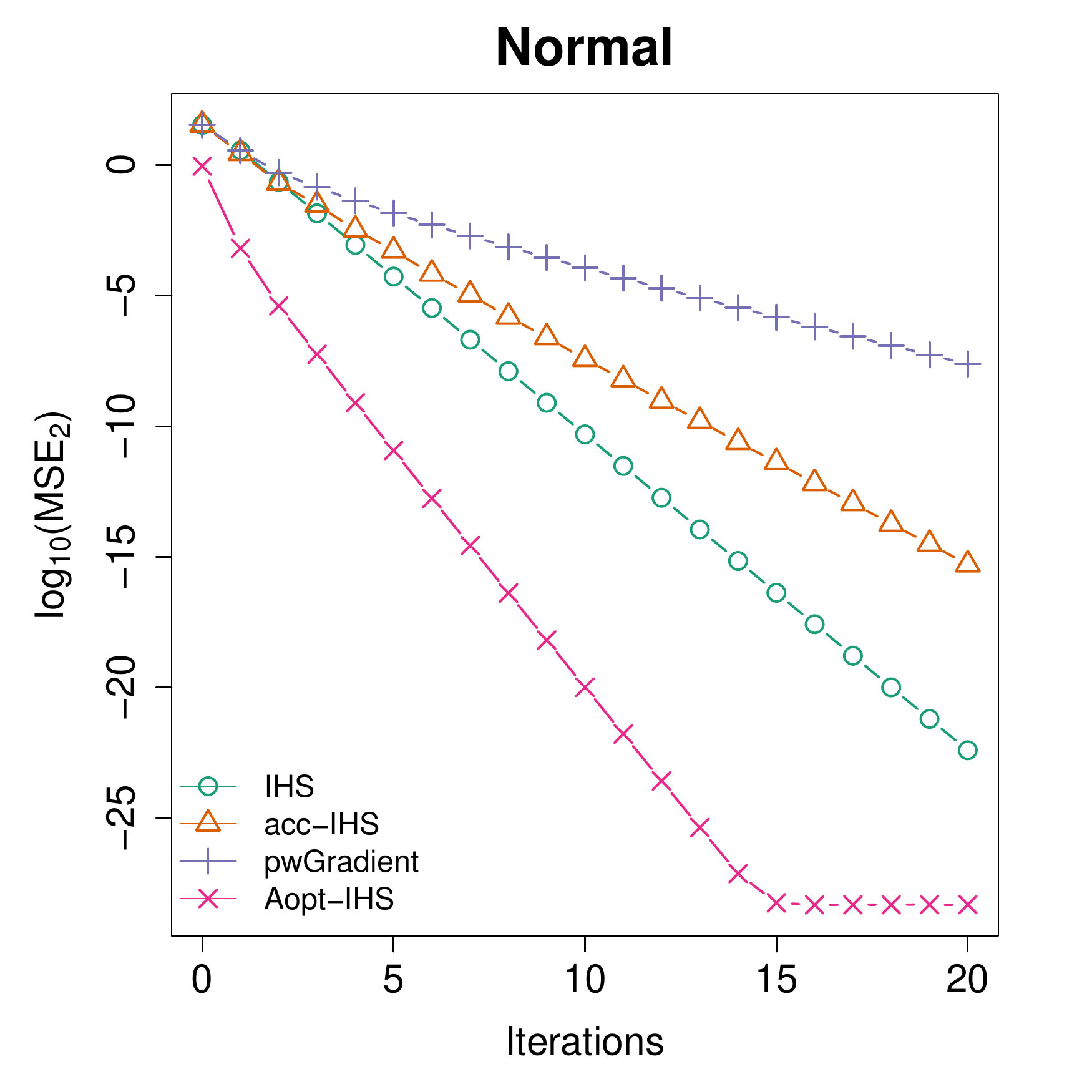}
	\end{minipage}}
	\subfigure[Log-Normal]{
		\begin{minipage}[b]{0.4\linewidth}
			\includegraphics[width=\linewidth]{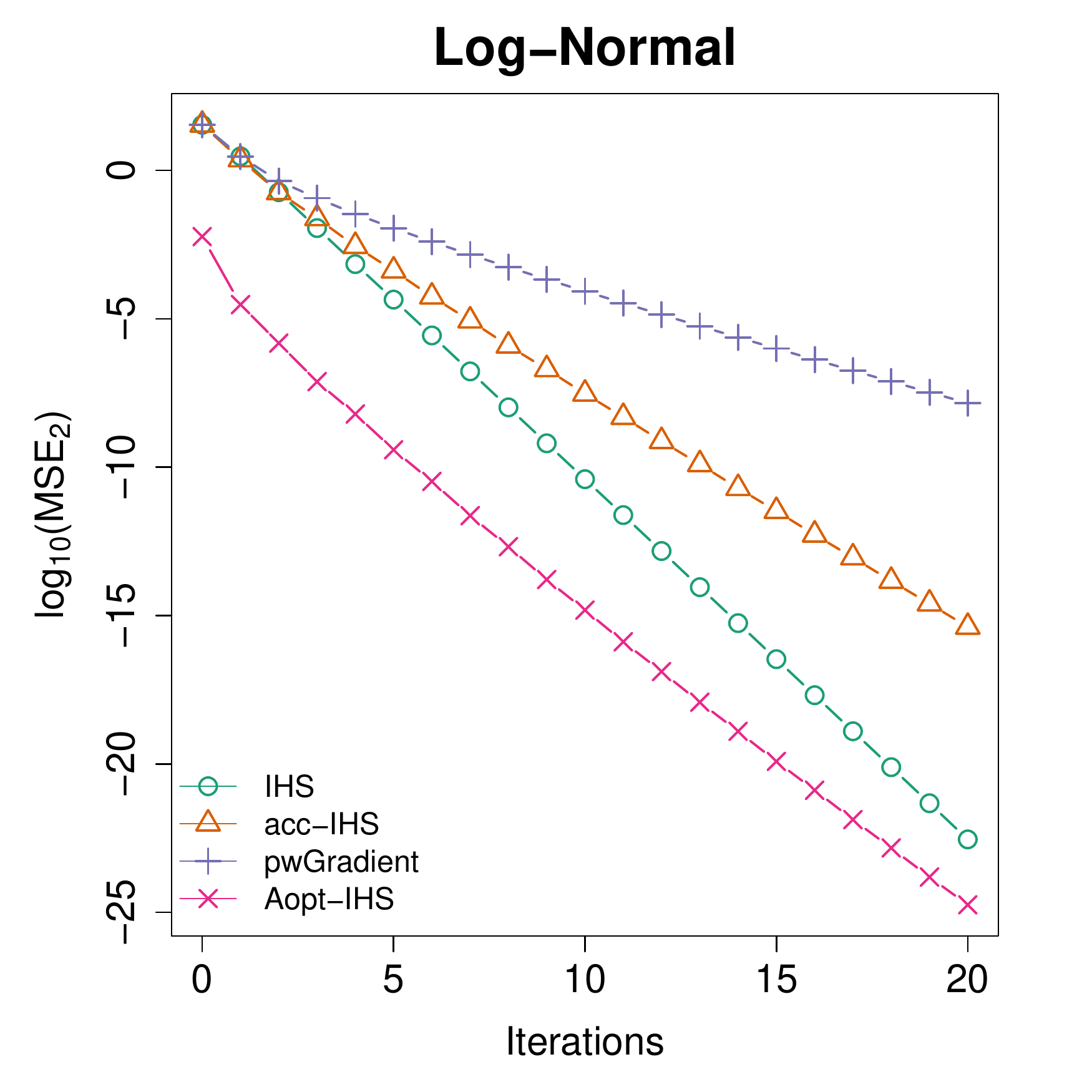}
	\end{minipage}}
	\subfigure[$t_2$]{
		\begin{minipage}[b]{0.4\linewidth}
			\includegraphics[width=\linewidth]{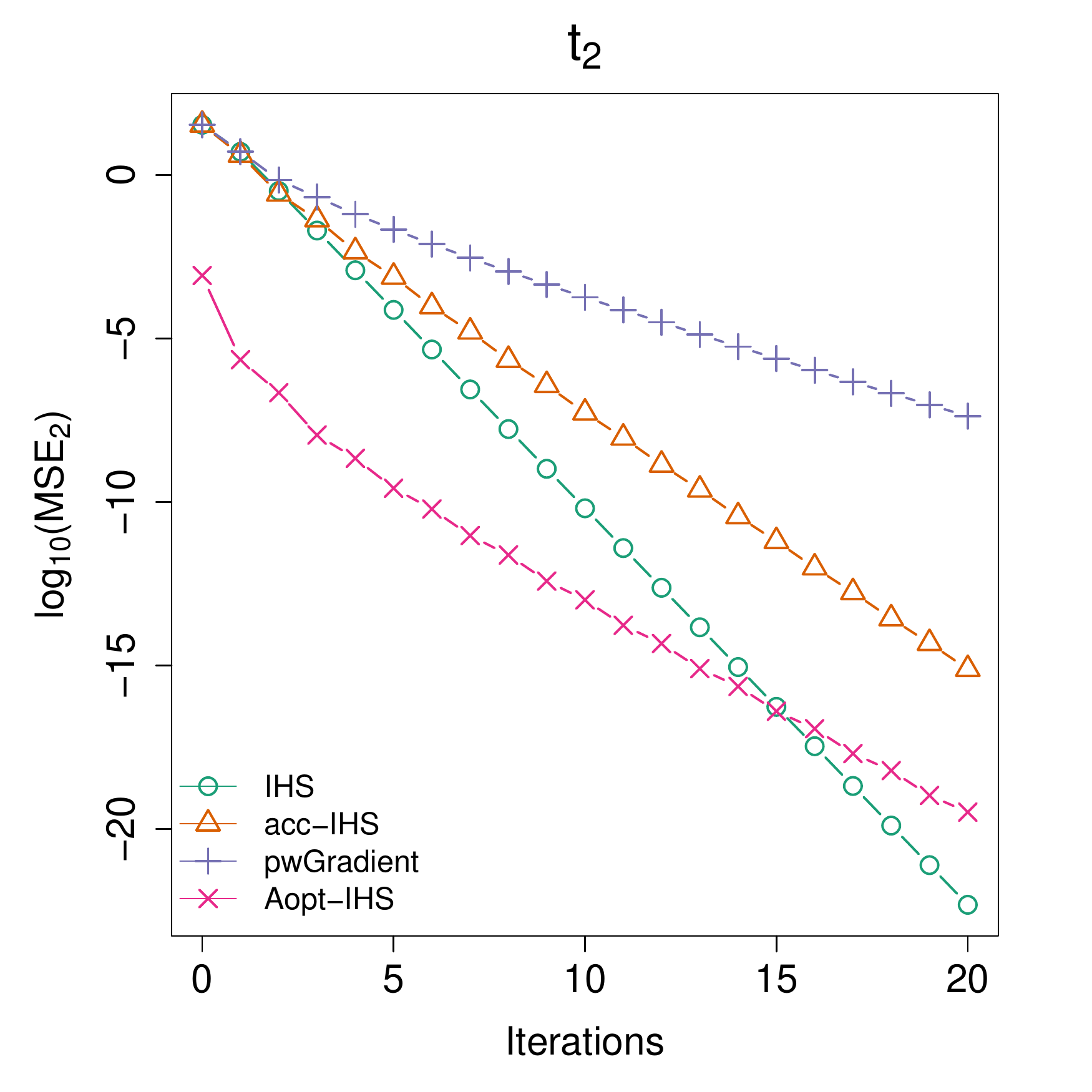}
	\end{minipage}}
	\subfigure[Mixture]{
		\begin{minipage}[b]{0.4\linewidth}
			\includegraphics[width=\linewidth]{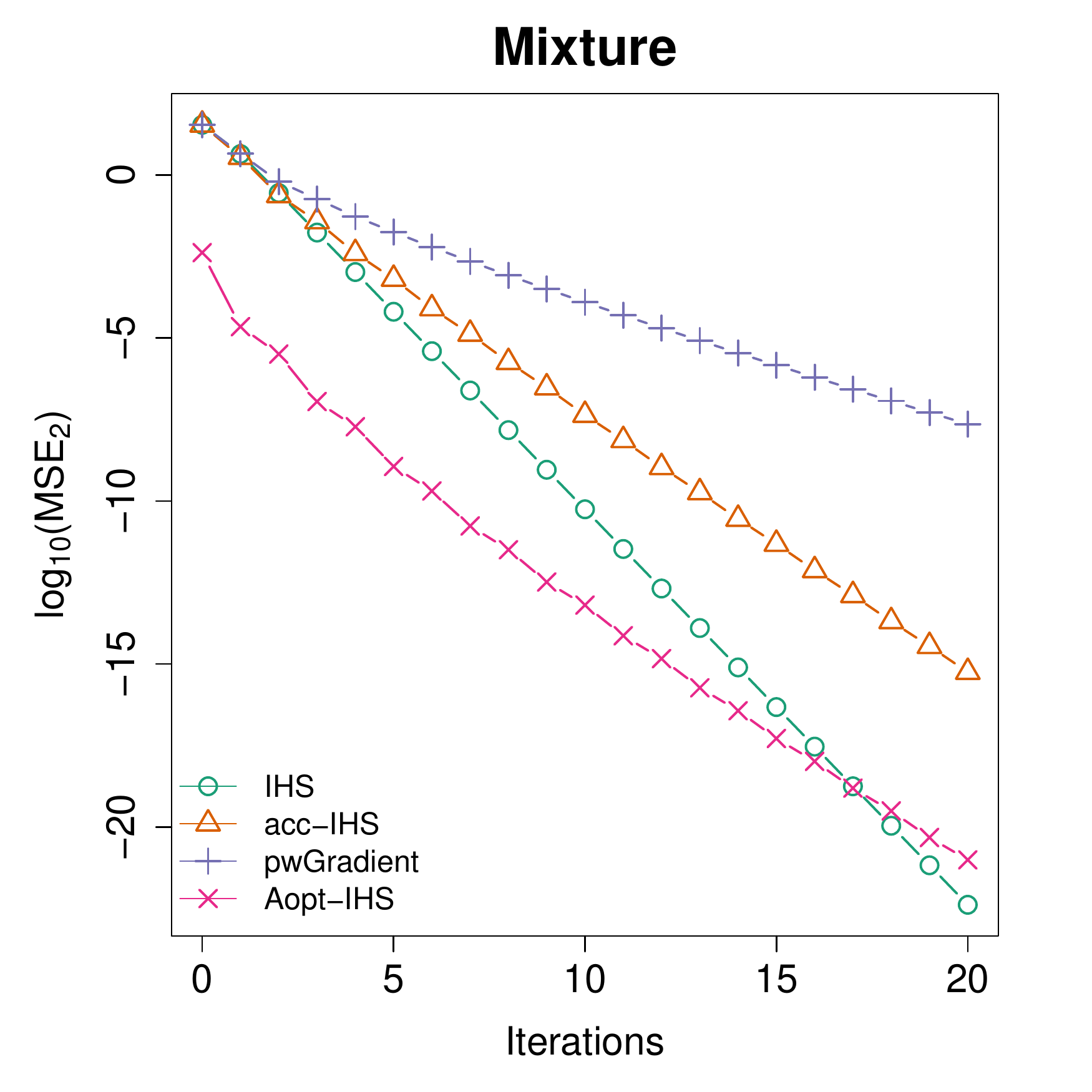}
	\end{minipage}}
	\begin{mycaption}\label{fig-lse}
	Estimation error between the estimator and the LSE based on full data versus iteration number $N$ when $d=50$.
	\end{mycaption}
\end{figure}

Figures \ref{fig-true} and \ref{fig-lse} illustrate the advantages of our proposed Aopt-IHS method. Firstly, the superiority of initialization can be thoroughly reflected by the flat curve in Figure~\ref{fig-true}, as well as the lower starting values in Figure~\ref{fig-lse}. During the iterations, the outperformance of the proposed method can be reflected in Figure \ref{fig-lse}. Specifically, \revise{given the convergent property of $\hat{\bm{\beta}}_t$, the curve slope reflects the convergence rate while the intercept or the starting point measures the effect of the initialization. From the plots, the proposed Aopt-IHS method not only outperforms in the initialization but also demonstrates a decent convergence rate that is competitive to pwGradient and it is even sharper than IHS for the normal case.} 

To compare the computational time, we present in Table~\ref{tb2} the averaged run time for each method for achieving the precision $10^{-10}$ in $\|\hat{\bm{\beta}}-\hat{\bm{\beta}}^{\rm LS}\|_2$, as well as the required numbers of iterations. Note that the time for computing $\bm{X}^T\bm{y}$ in the update formula is not counted for all methods. Table~\ref{tb2} shows that our method can attain the same precision with fewer iterations and shorter time. 

\begin{table}[ht]
\centering
\caption{The averaged time for each approach achieving the precision $10^{-10}$ measured by $\|\hat{\bm{\beta}}-\hat{\bm{\beta}}^{\rm LS}\|_2$, and the required  numbers of iterations.}\label{tb2}
\setlength{\tabcolsep}{3pt} 
\begin{tabular}{{l}*{8}{c}}
			\toprule
			\multirow{4}*{Method} & \multicolumn{2}{c}{Normal} &
			\multicolumn{2}{c}{Log-Normal} &  \multicolumn{2}{c}{$t_2$} &  \multicolumn{2}{c}{Mixture} \\
			\cmidrule(lr){2-3}\cmidrule(lr){4-5}\cmidrule(lr){6-7}\cmidrule(lr){8-9}
			&Time(s)& Iter & Time(s)& Iter& Time(s)& Iter &Time(s)& Iter\\
			\cmidrule(lr){2-9}
			
			&\multicolumn{8}{c}{$d=50$}\\
			IHS        & 15.28 & 18.30   & 14.83 & 18.25   & 15.63 & 18.43   & 15.21 & 18.38   \\
			acc-IHS    & 5.63  & 25.95   & 5.20  & 25.87   & 5.57  & 26.17   & 5.88  & 26.05   \\
			pwGradient & 8.54  & 46.99   & 8.19  & 47.40   & 8.98  & 47.34   & 9.22  & 47.67   \\
			Aopt-IHS   & 3.01  & 10.27   & 3.90  & 14.97   & 3.75  & 12.65   & 4.74  & 17.39   \\
			
			&\multicolumn{8}{c}{$d=100$}\\
		IHS        & 53.55 & 27.29   & 54.89 & 27.16   & 48.17 & 27.51   & 43.31 & 27.40   \\
		acc-IHS    & 19.59 & 40.64   & 19.59 & 40.51   & 17.11 & 41.01   & 15.42 & 40.92   \\
		pwGradient & --- & diverge & --- & diverge & --- & diverge & --- & diverge \\
		Aopt-IHS   & 11.28 & 19.44   & 11.16 & 19.07   & 11.32 & 22.78   & 9.48  & 20.45 \\
		\bottomrule
	\end{tabular}
\end{table}

\revise{
\subsection{Real Data Analysis}
Finally, we apply the proposed method to a real food intake dataset\footnote{Data can be found in \url{https://www.icpsr.umich.edu/icpsrweb/ICPSR/studies/21960#}.}. This data, consisting of three sub-files, records the information on one-day dietary intakes of male residents in the United States with age from 19 to 50 in 1985. We consider the largest food intake file that  includes how much and when the food items are taken. This sub-file has $n=10177$ samples and we are interested in the linear regression model between the food energy ($y$) and the four features: total protein ($x_1$), fat ($x_2$), carbohydrate ($x_3$), and alcohol ($x_4$). Both the feature and target variables are centered from we feed them to the linear model.}

\revise{Using the IHS methods with sketch size as $m=500$ and  $\lambda=0.1\sum_{i=1}^n\|\bm{x}_i\|^2_2$, we evaluate their performances through the Euclidean distance between the estimated parameter  and the full data least square estimator $\|\hat{\bm{\beta}}-\hat{\bm{\beta}}^{\rm LS}\|_2$. Same as the simulation study above, we take the trimmed average of Euclidean errors over $R=1000$ repetitions for the randomized methods. The comparative results are plotted in Figure~\ref{fig-fi} and and it is shown that our proposed Aopt-IHS has the best performance during the beginning iterations. It outperforms both acc-IHS and pwGradient for iterations up to 20. In this real data analysis, the randomized IHS  has the best convergence rate, but computationally it is the least efficient method. 
\begin{figure}[ht!]
\centering
\includegraphics[width=0.6\linewidth]{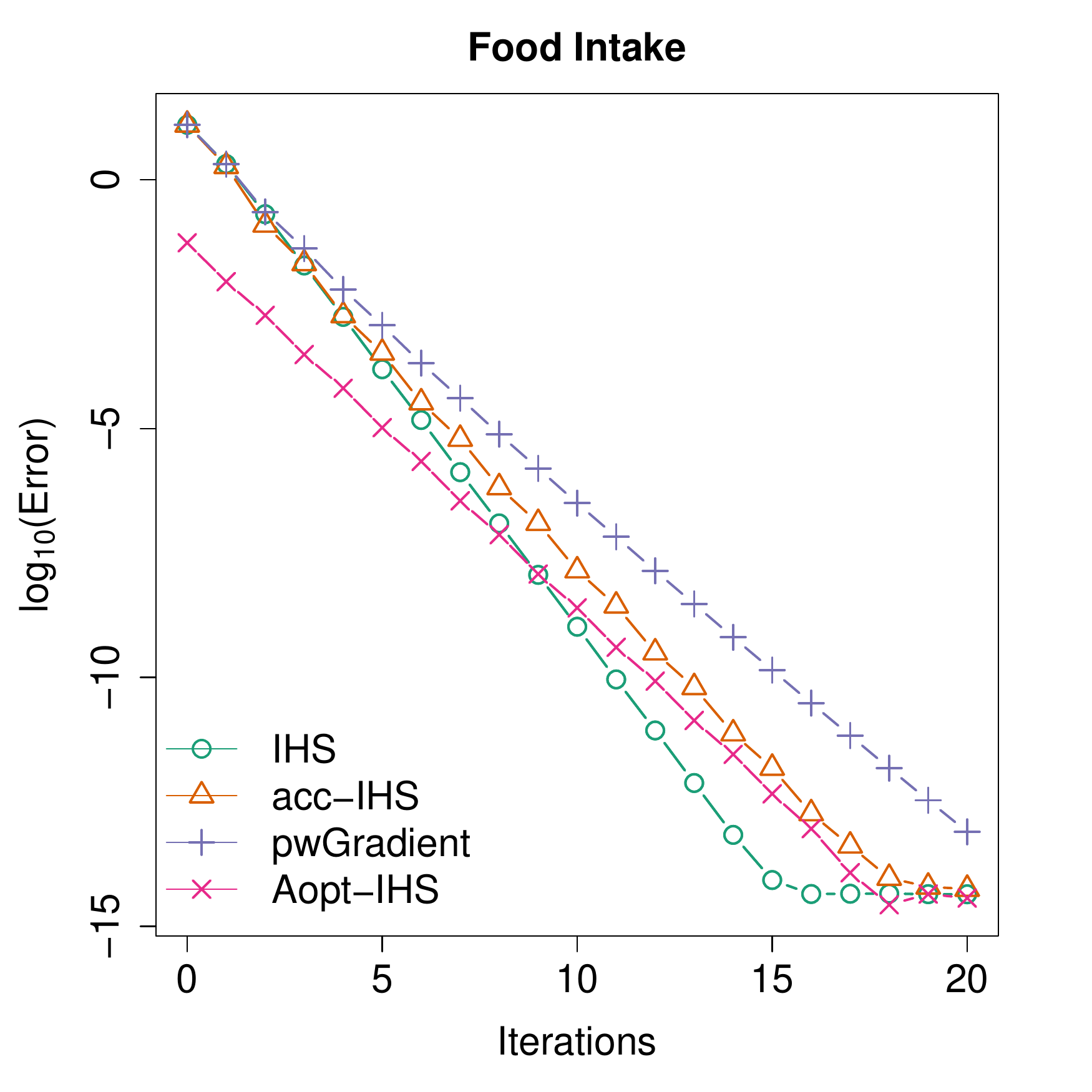}
\begin{mycaption}\label{fig-fi}
The approximation error $\|\hat{\bm{\beta}}-\hat{\bm{\beta}}^{\rm LS}\|_2$ of different IHS methods on the food intake data with sketch size $m=500$.  \end{mycaption}
\end{figure}
}

\section{Conclusion}
We reformulate the IHS method as an adaptive first-order optimization method, by using the idea of optimal design of experiments for subdata selection. To our best knowledge, this is the first attempt to improve the IHS method in a deterministic manner while maintaining a decent speed and precision. The numerical experiments confirm the superiority of the proposed approach.

There are several open problems worth of further investigation, including the theoretical properties of the ridged preconditioner according to the conditions derived in Theorem~\ref{thm1}. \revise{It is of our future interest to investigate how the ridge term affects the convergence rate of Algorithm~3.} Other than using the $A$-optimal design, it is also interesting to investigate the $D$-optimal or other types of optimal designs for the purpose of subdata selection, in particular when the data are heterogeneously distributed.



\bibliographystyle{apalike}       
\bibliography{AdaIHS}   

\begin{thebibliography}{}

\bibitem[Benzi, 2002]{benzi2002preconditioning}
Benzi, M. (2002).
\newblock Preconditioning techniques for large linear systems: A survey.
\newblock {\em Journal of Computational Physics}, 182(2):418--477.

\bibitem[Boutsidis and Gittens, 2013]{boutsidis2013improved}
Boutsidis, C. and Gittens, A. (2013).
\newblock Improved matrix algorithms via the subsampled randomized hadamard
  transform.
\newblock {\em SIAM Journal on Matrix Analysis and Applications},
  34(3):1301--1340.

\bibitem[Clarkson and Woodruff, 2013]{clarkson2013low}
Clarkson, K.~L. and Woodruff, D.~P. (2013).
\newblock Low rank approximation and regression in input sparsity time.
\newblock In {\em Proceedings of the Forty-Fifth Annual ACM Symposium on Theory
  of Computing}, pages 81--90. ACM.

\bibitem[Drineas et~al., 2012]{drineas2012fast}
Drineas, P., Magdon-Ismail, M., Mahoney, M.~W., and Woodruff, D.~P. (2012).
\newblock Fast approximation of matrix coherence and statistical leverage.
\newblock {\em Journal of Machine Learning Research}, 13(Dec):3475--3506.

\bibitem[Drineas et~al., 2006]{drineas2006sampling}
Drineas, P., Mahoney, M.~W., and Muthukrishnan, S. (2006).
\newblock Sampling algorithms for l 2 regression and applications.
\newblock In {\em Proceedings of the Seventeenth Annual ACM-SIAM Symposium on
  Discrete Algorithm}, pages 1127--1136. Society for Industrial and Applied
  Mathematics.

\bibitem[Drineas et~al., 2011]{drineas2011faster}
Drineas, P., Mahoney, M.~W., Muthukrishnan, S., and Sarl{\'o}s, T. (2011).
\newblock Faster least squares approximation.
\newblock {\em Numerische Mathematik}, 117(2):219--249.

\bibitem[Gonen et~al., 2016]{gonen2016solving}
Gonen, A., Orabona, F., and Shalev-Shwartz, S. (2016).
\newblock Solving ridge regression using sketched preconditioned svrg.
\newblock In {\em International Conference on Machine Learning}, pages
  1397--1405.

\bibitem[Horn and Johnson, 2012]{horn2012matrix}
Horn, R.~A. and Johnson, C.~R. (2012).
\newblock {\em Matrix analysis}.
\newblock Cambridge university press.

\bibitem[Johnson and Lindenstrauss, 1984]{johnson1984extensions}
Johnson, W.~B. and Lindenstrauss, J. (1984).
\newblock Extensions of lipschitz mappings into a hilbert space.
\newblock {\em Contemporary Mathematics}, 26(189-206):1.

\bibitem[Knyazev and Lashuk, 2007]{knyazev2007steepest}
Knyazev, A.~V. and Lashuk, I. (2007).
\newblock Steepest descent and conjugate gradient methods with variable
  preconditioning.
\newblock {\em SIAM Journal on Matrix Analysis and Applications},
  29(4):1267--1280.

\bibitem[Lu et~al., 2013]{lu2013faster}
Lu, Y., Dhillon, P., Foster, D.~P., and Ungar, L. (2013).
\newblock Faster ridge regression via the subsampled randomized hadamard
  transform.
\newblock In {\em Advances in Neural Information Processing Systems}, pages
  369--377.

\bibitem[Ma et~al., 2015]{ma2015statistical}
Ma, P., Mahoney, M.~W., and Yu, B. (2015).
\newblock A statistical perspective on algorithmic leveraging.
\newblock {\em The Journal of Machine Learning Research}, 16(1):861--911.

\bibitem[Mahoney et~al., 2011]{mahoney2011randomized}
Mahoney, M.~W. et~al. (2011).
\newblock Randomized algorithms for matrices and data.
\newblock {\em Foundations and Trends{\textregistered} in Machine Learning},
  3(2):123--224.

\bibitem[Mart{\i}nez, 2004]{martinez2004partial}
Mart{\i}nez, C. (2004).
\newblock Partial quicksort.
\newblock In {\em Proc. 6th ACMSIAM Workshop on Algorithm Engineering and
  Experiments and 1st ACM-SIAM Workshop on Analytic Algorithmics and
  Combinatorics}, pages 224--228.

\bibitem[McWilliams et~al., 2014]{mcwilliams2014fast}
McWilliams, B., Krummenacher, G., Lucic, M., and Buhmann, J.~M. (2014).
\newblock Fast and robust least squares estimation in corrupted linear models.
\newblock In {\em Advances in Neural Information Processing Systems}, pages
  415--423.

\bibitem[Nocedal and Wright, 2006]{Nocedal2006Numerical}
Nocedal, J. and Wright, S.~J. (2006).
\newblock {\em Numerical Optimization}.
\newblock Springer.

\bibitem[Pilanci and Wainwright, 2016]{pilanci2016iterative}
Pilanci, M. and Wainwright, M.~J. (2016).
\newblock Iterative hessian sketch: Fast and accurate solution approximation
  for constrained least-squares.
\newblock {\em The Journal of Machine Learning Research}, 17(1):1842--1879.

\bibitem[Pukelsheim, 1993]{pukelsheim1993optimal}
Pukelsheim, F. (1993).
\newblock {\em Optimal Design of Experiments}, volume~50.
\newblock SIAM.

\bibitem[Seber, 2008]{seber2008matrix}
Seber, G.~A. (2008).
\newblock {\em A matrix handbook for statisticians}, volume~15.
\newblock John Wiley \& Sons.

\bibitem[Tropp, 2011]{tropp2011improved}
Tropp, J.~A. (2011).
\newblock Improved analysis of the subsampled randomized hadamard transform.
\newblock {\em Advances in Adaptive Data Analysis}, 3(1--2):115--126.

\bibitem[Wang and Xu, 2018]{wang2018large}
Wang, D. and Xu, J. (2018).
\newblock Large scale constrained linear regression revisited: Faster
  algorithms via preconditioning.
\newblock In {\em Thirty-Second AAAI Conference on Artificial Intelligence}.

\bibitem[Wang et~al., 2018]{wang2018information}
Wang, H., Yang, M., and Stufken, J. (2018).
\newblock Information-based optimal subdata selection for big data linear
  regression.
\newblock {\em Journal of the American Statistical Association}, pages 1--13.

\bibitem[Wang et~al., 2017]{wang2017sketching}
Wang, J., Lee, J.~D., Mahdavi, M., Kolar, M., Srebro, N., et~al. (2017).
\newblock Sketching meets random projection in the dual: A provable recovery
  algorithm for big and high-dimensional data.
\newblock {\em Electronic Journal of Statistics}, 11(2):4896--4944.

\bibitem[Woodruff et~al., 2014]{woodruff2014sketching}
Woodruff, D.~P. et~al. (2014).
\newblock Sketching as a tool for numerical linear algebra.
\newblock {\em Foundations and Trends{\textregistered} in Theoretical Computer
  Science}, 10(1--2):1--157.

\end{thebibliography}


\newpage
\appendix

\section*{Appendix A: Extra Results with Different Ridged Preconditioners}
We further compare the ridged preconditioner $\bm{M}=\frac{n}{m}\sum_{i=1}^n\delta_i\bm{x}_i\bm{x}_i^T+\lambda\bm{I}_d$ with its two components, the non-ridged term and the scaled identity matrix. Three preconditioners are evaluated through ${\rm MSE}_2$ under our proposed algorithm framework. We only consider the identity matrix $\bm{M}= \bm{I}$ since any scaling multiplier $\lambda$ in $\bm{M}= \lambda\bm{I}$ can be canceled out during the update of $\hat{\bm{\beta}}_t$. The results strengthen that the ridging operation may enhance the preconditioner performance.

\begin{figure}[ht!]
	\centering
	\setcounter{subfigure}{0}
	\subfigure[Normal]{
	\begin{minipage}[b]{0.40\linewidth}
			\centering
			\includegraphics[width=\linewidth]{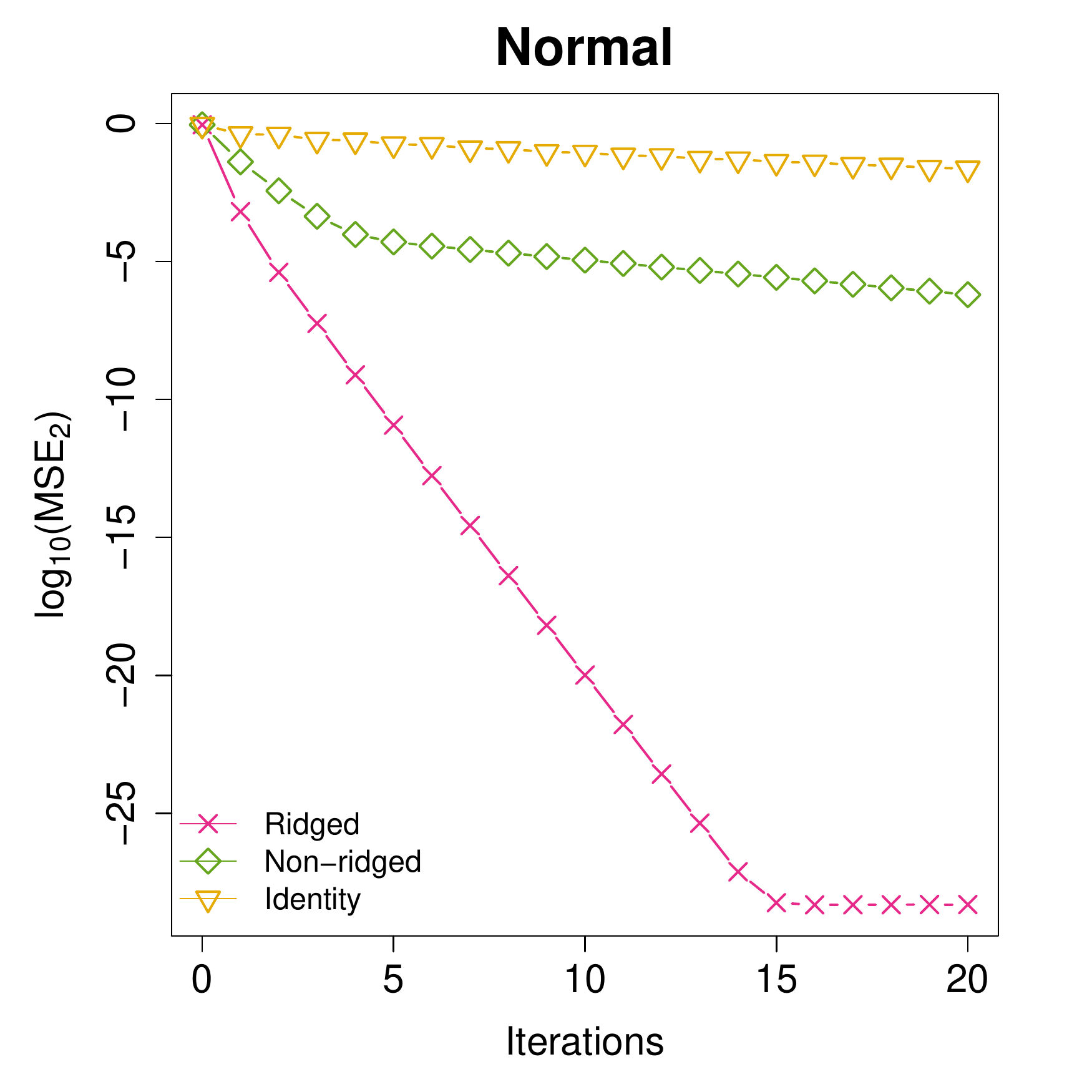}
	\end{minipage}}
	\subfigure[Log-Normal]{
	\begin{minipage}[b]{0.40\linewidth}
			\centering
			\includegraphics[width=\linewidth]{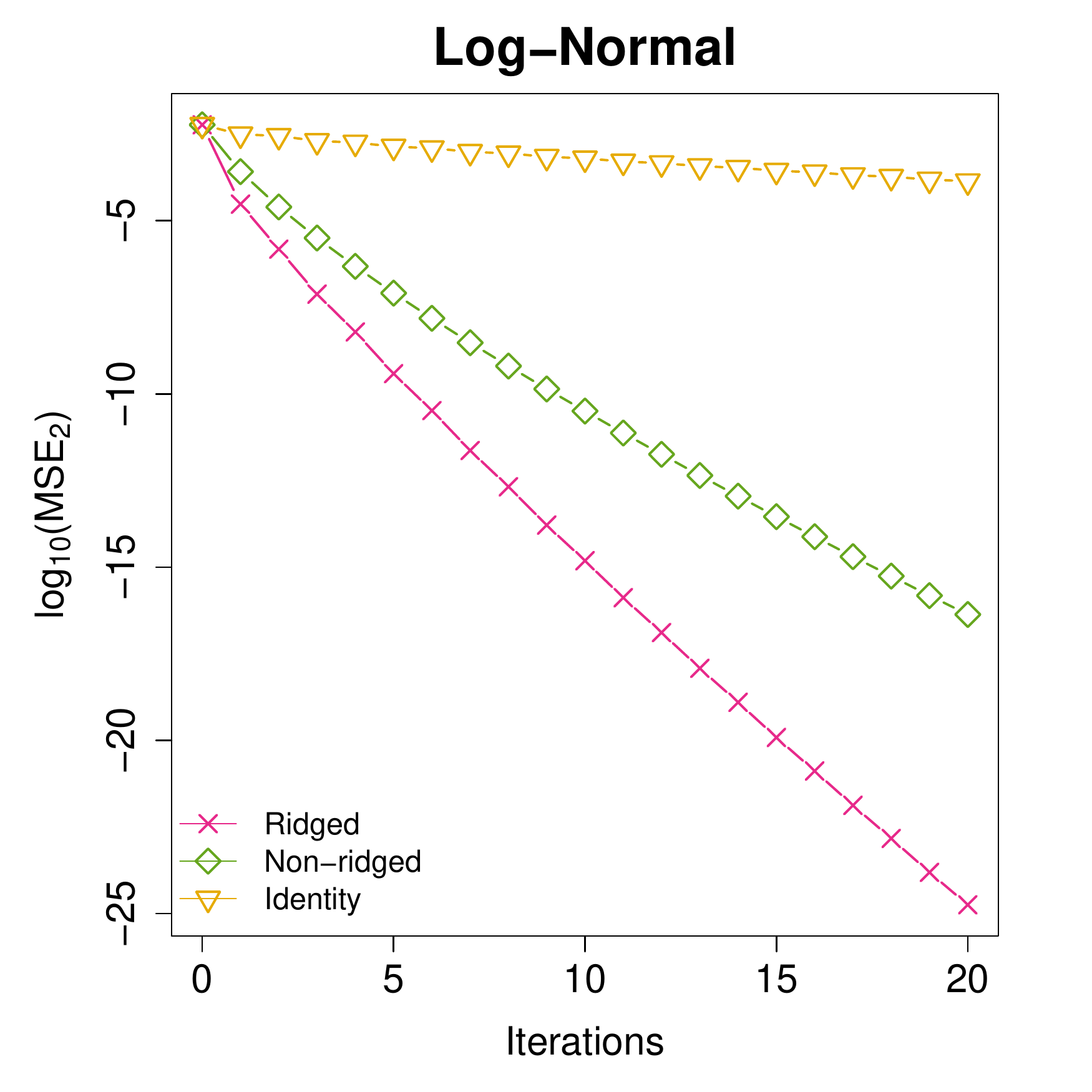}
	\end{minipage}}
	\subfigure[$t_2$]{
		\begin{minipage}[b]{0.40\linewidth}
		\centering
			\includegraphics[width=\linewidth]{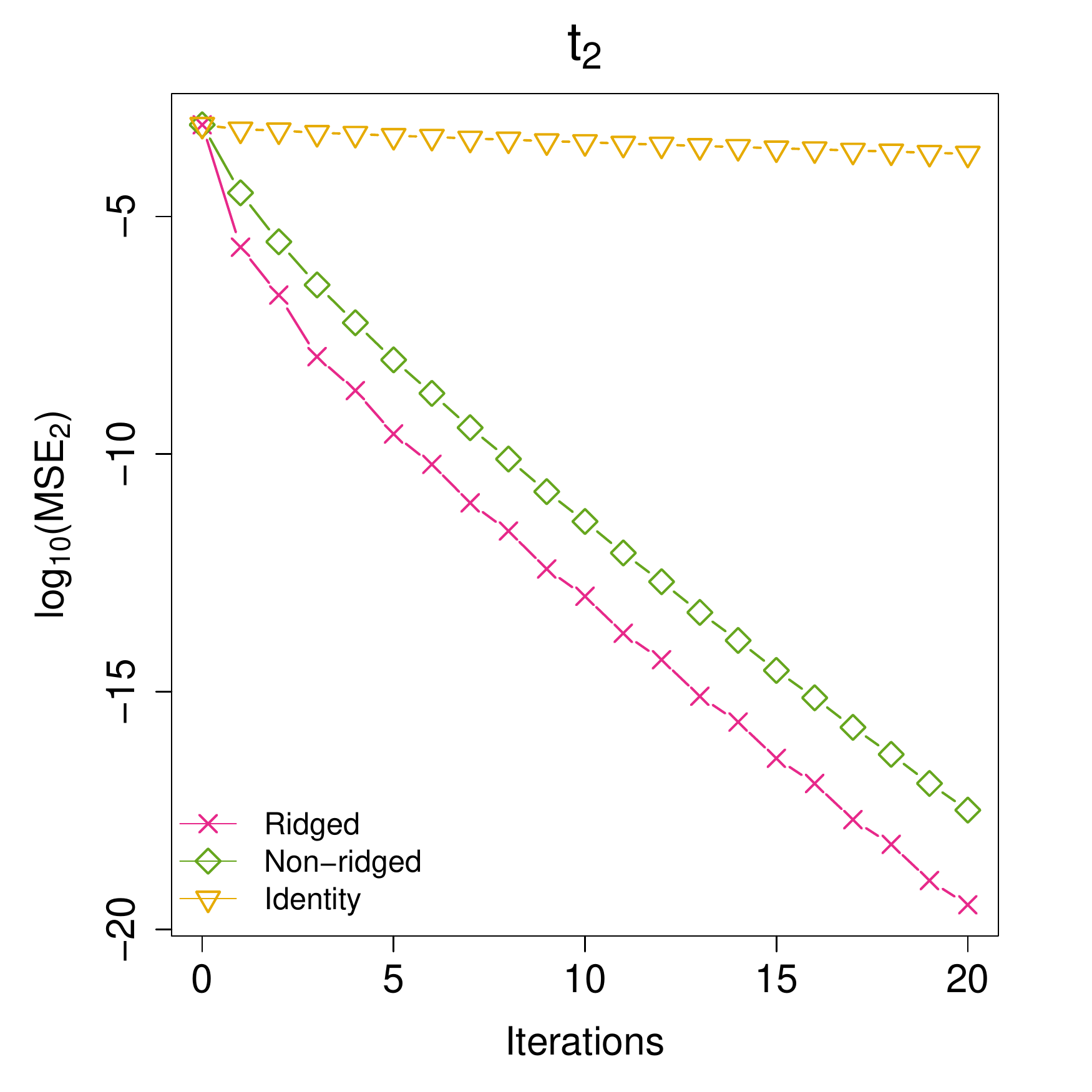}
	\end{minipage}}
	\subfigure[Mixture]{
		\begin{minipage}[b]{0.40\linewidth}
		\centering
			\includegraphics[width=\linewidth]{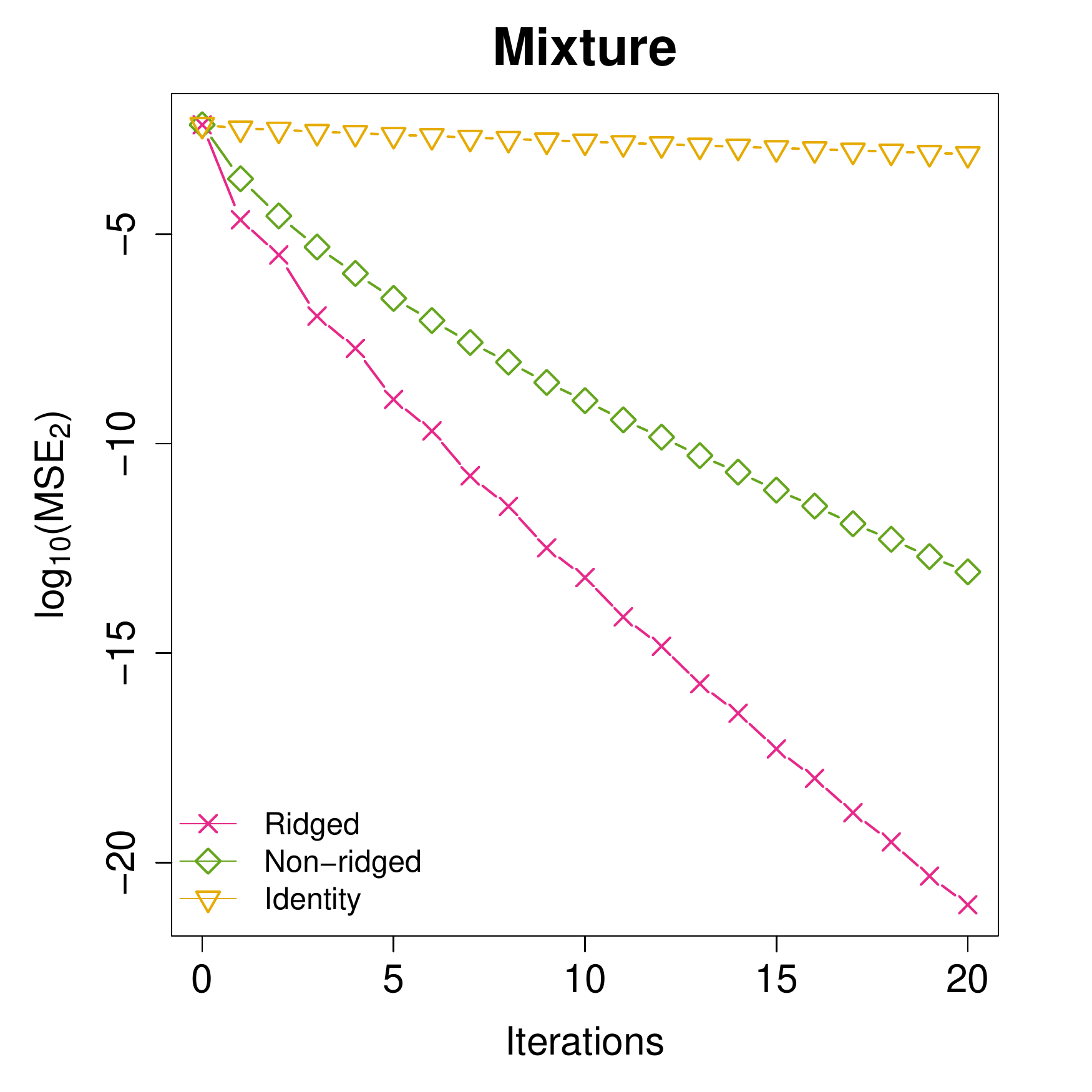}
	\end{minipage}}
	\begin{mycaption}\label{fig-ridge}
	Estimation error between the estimator and the LSE versus iteration number $N$ for different preconditioners. 
	\end{mycaption}
\end{figure}

\newpage
\section*{Appendix B: Extra Results with the Same Proposed Initial Estimator}
\reviseb{In this section, we perform some additional experiments where all the methods are initialized by our proposed $A$-optimal estimator. The subsample size is fixed as $m=1000$. These experiments further justify that the proposed Aopt-IHS method generally enjoys the better convergent rates than the benchmark methods.}

\begin{figure}[ht!]
	\centering
	\setcounter{subfigure}{0}
	\subfigure[Normal]{
	\begin{minipage}[b]{0.4\linewidth}
			\includegraphics[width=\linewidth]{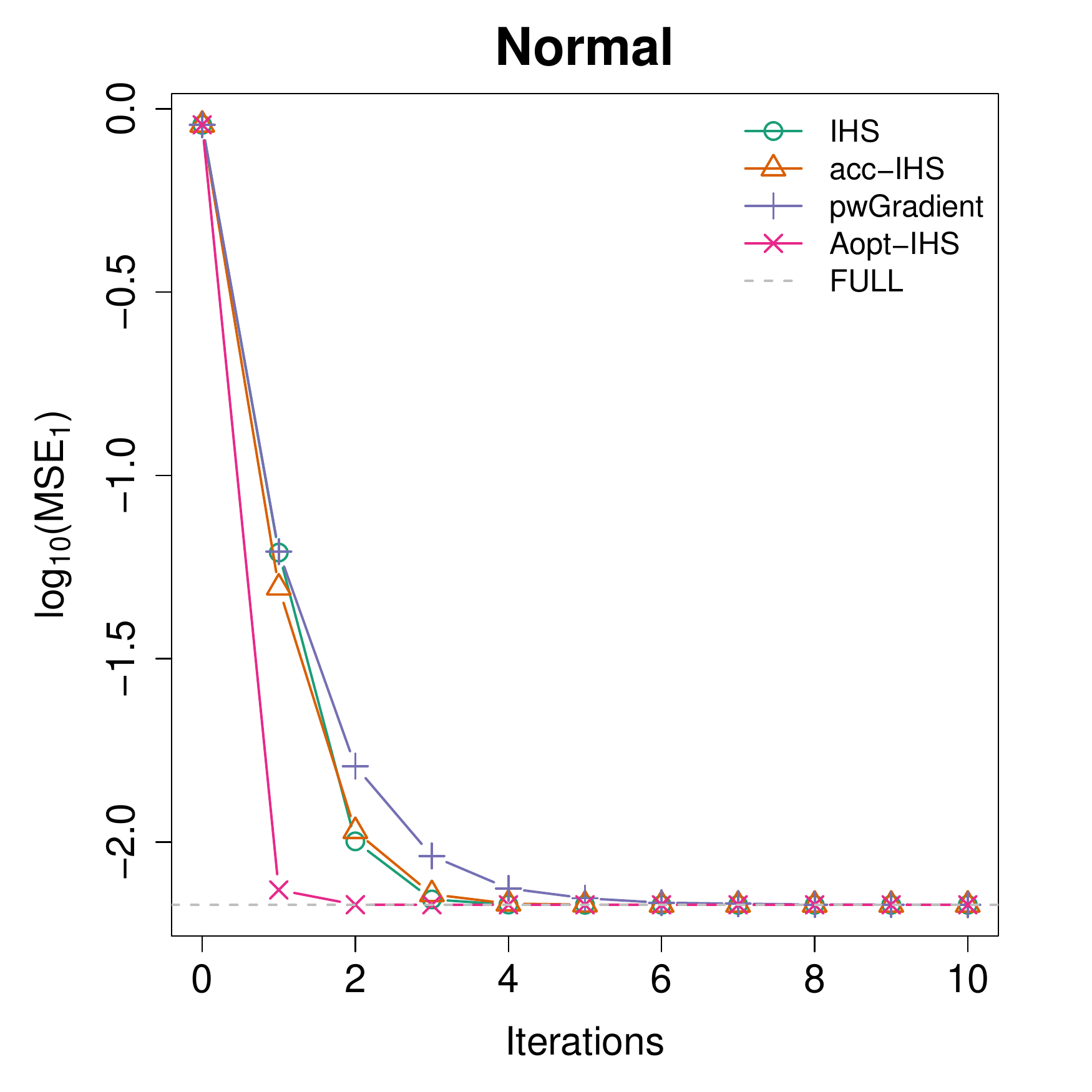}
	\end{minipage}}
	\subfigure[Log-Normal]{
		\begin{minipage}[b]{0.4\linewidth}
			\includegraphics[width=\linewidth]{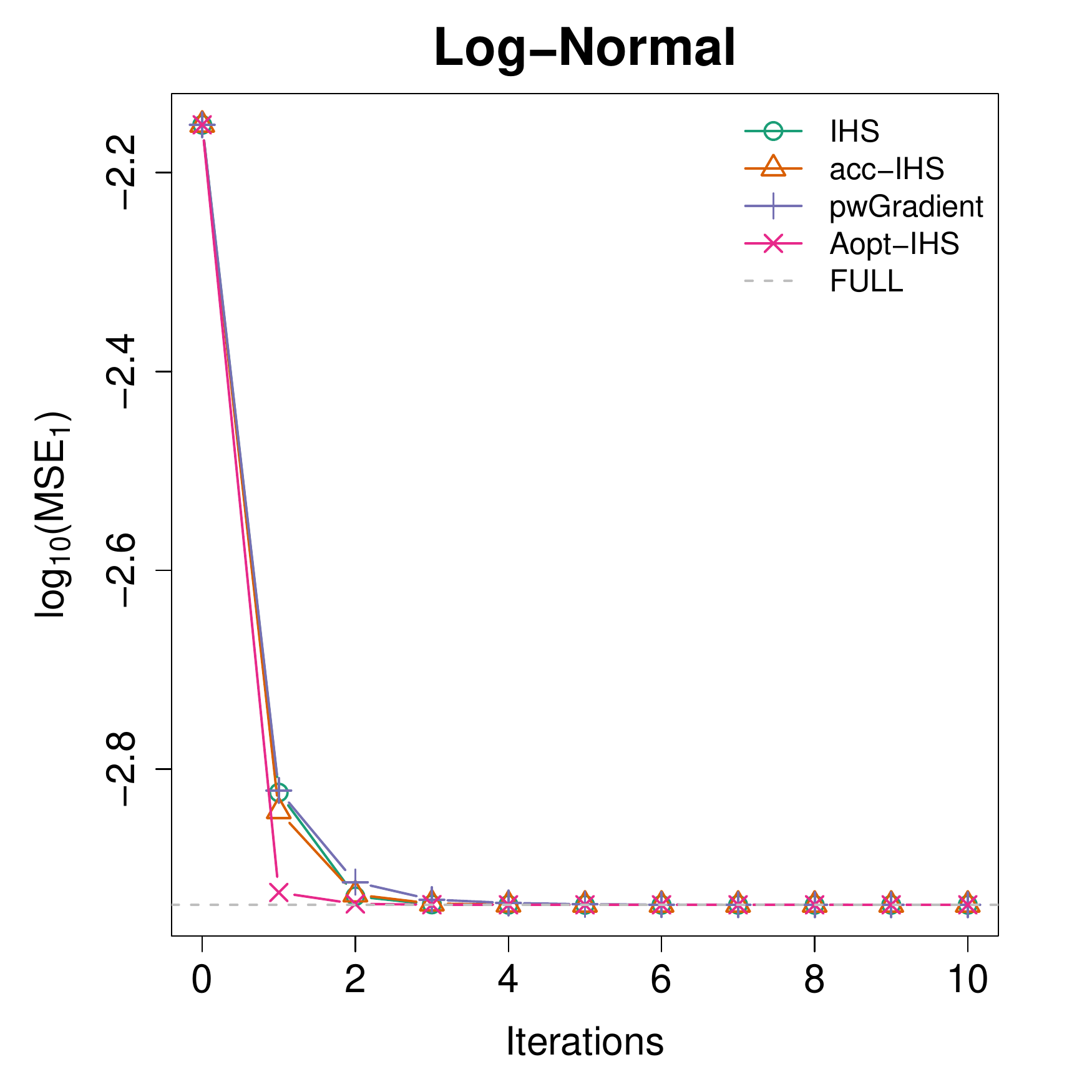}
	\end{minipage}}
	\subfigure[$t_2$]{
		\begin{minipage}[b]{0.4\linewidth}
			\includegraphics[width=\linewidth]{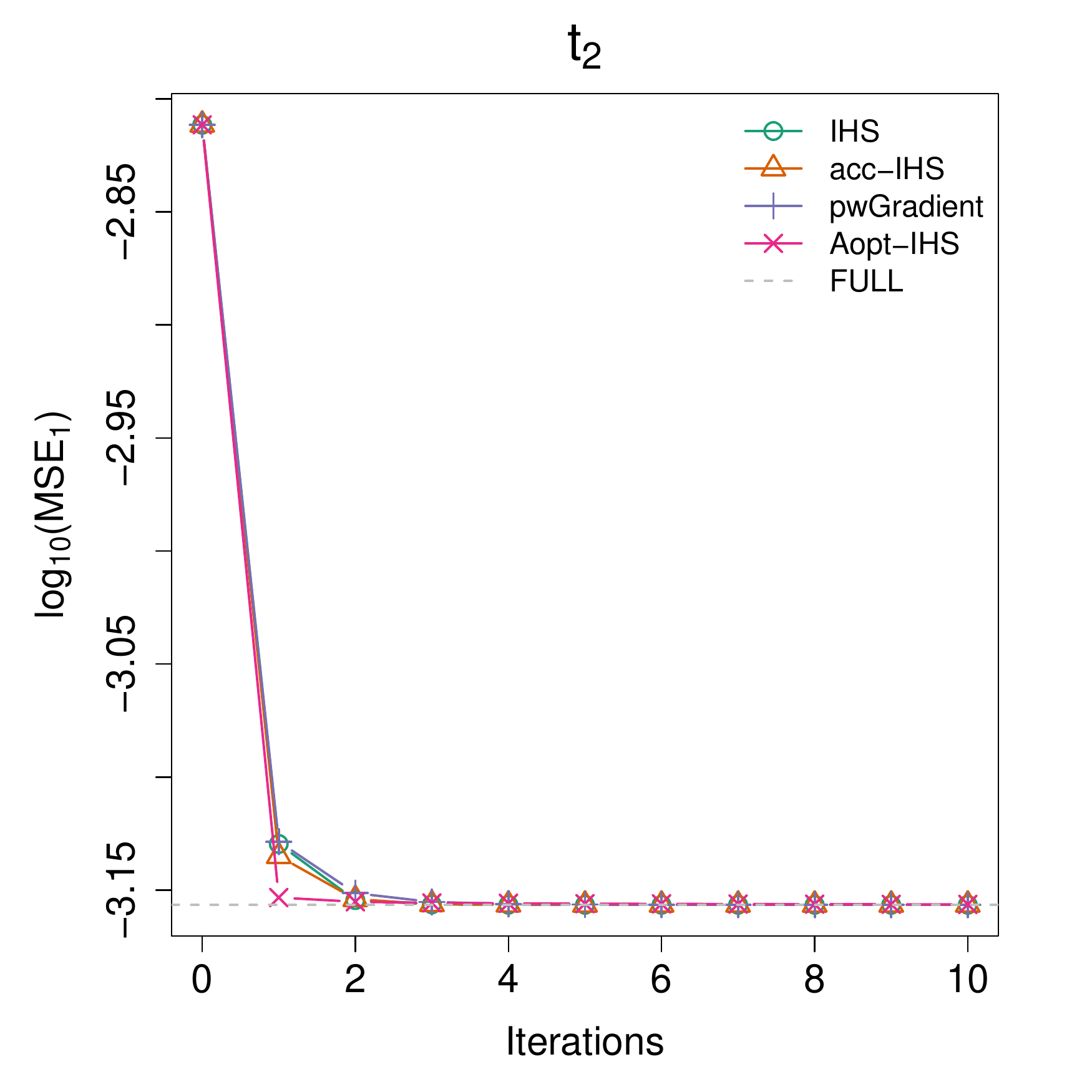}
	\end{minipage}}
	\subfigure[Mixture]{
		\begin{minipage}[b]{0.4\linewidth}
			\includegraphics[width=\linewidth]{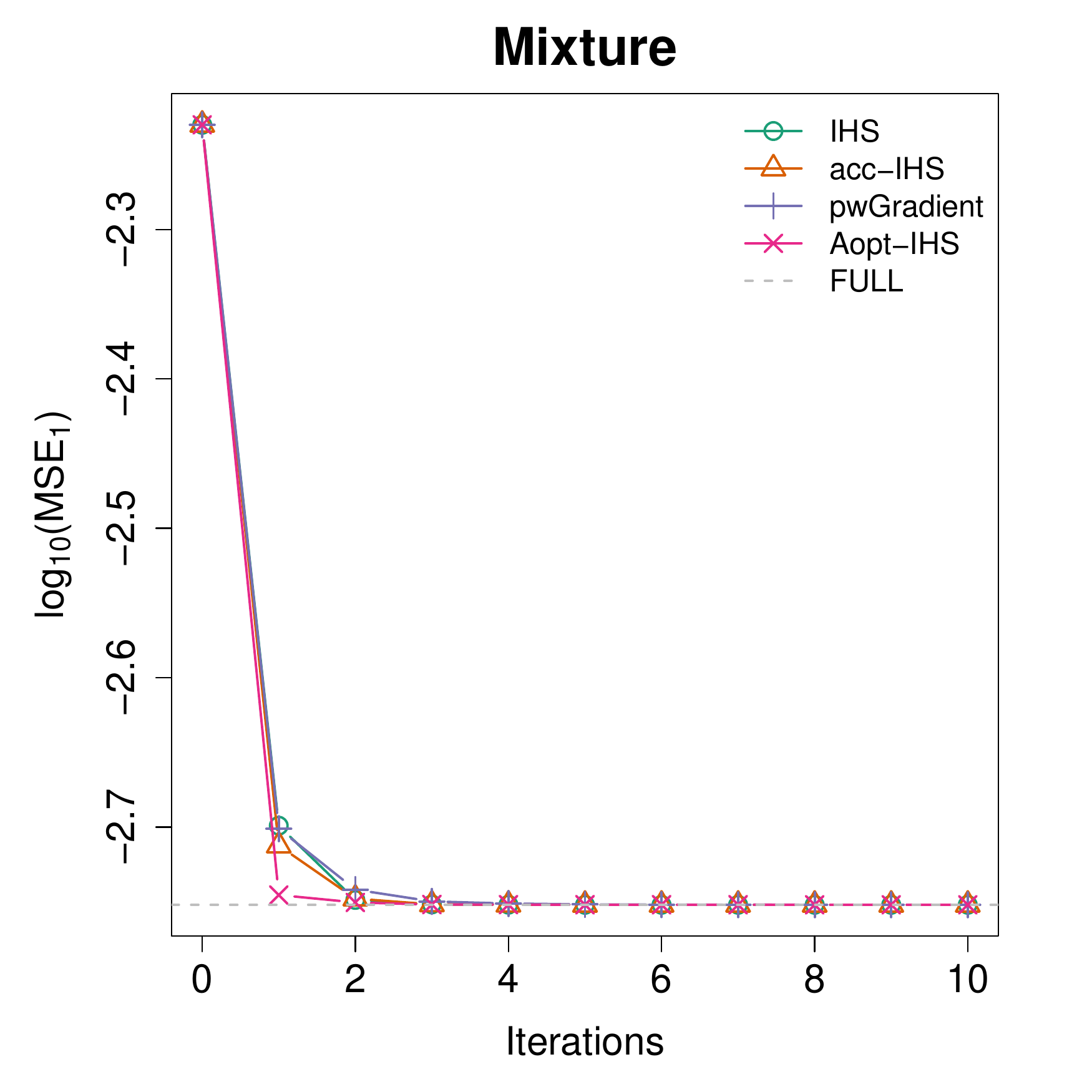}
	\end{minipage}}
	\begin{mycaption}
		Estimation error between the estimator and the ground truth versus iteration number $N$ when $d=50$.
	\end{mycaption}
\end{figure}

\begin{figure}[ht!]
	\centering
	\setcounter{subfigure}{0}
	\subfigure[Normal]{
	\begin{minipage}[b]{0.4\linewidth}
			\includegraphics[width=\linewidth]{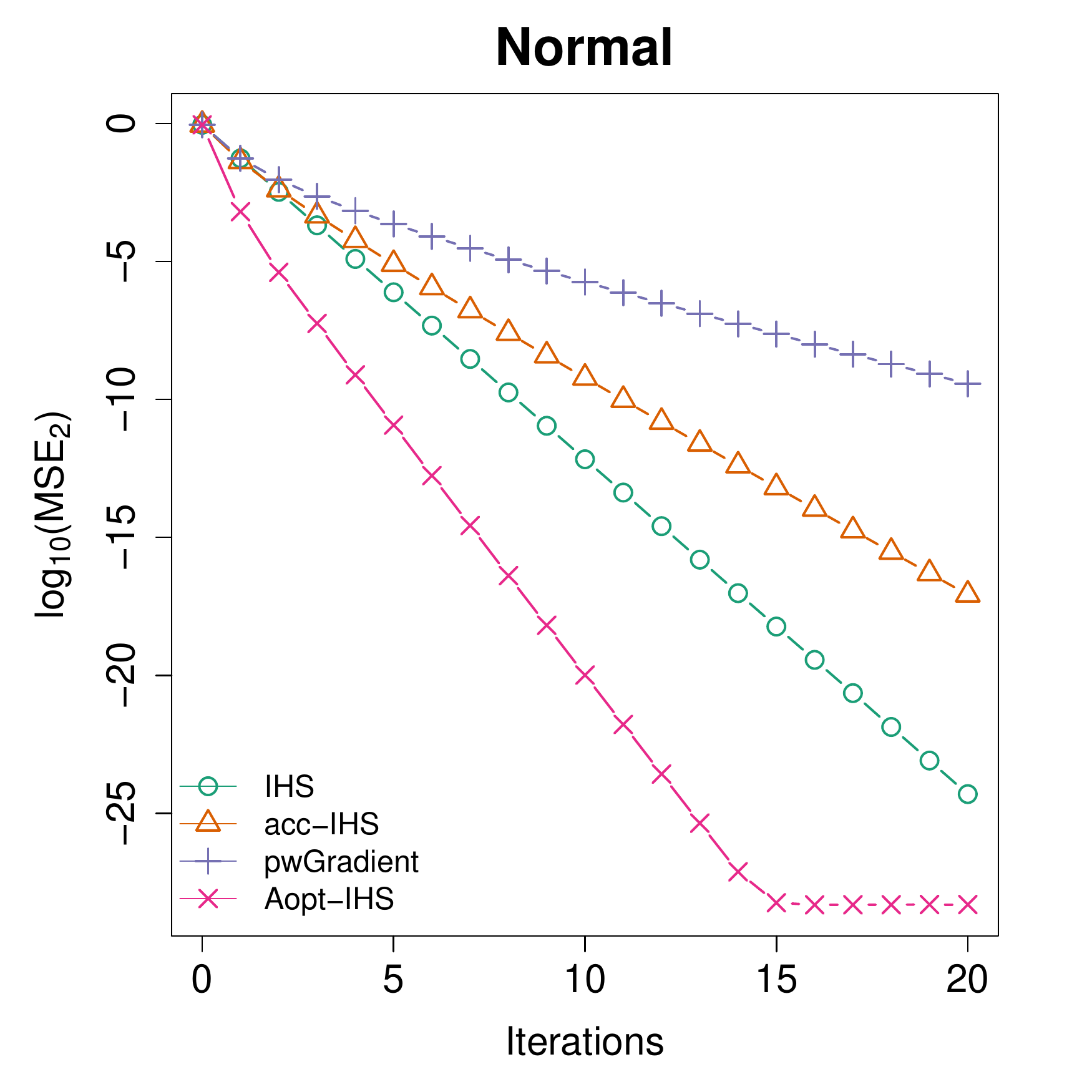}
	\end{minipage}}
	\subfigure[Log-Normal]{
		\begin{minipage}[b]{0.4\linewidth}
			\includegraphics[width=\linewidth]{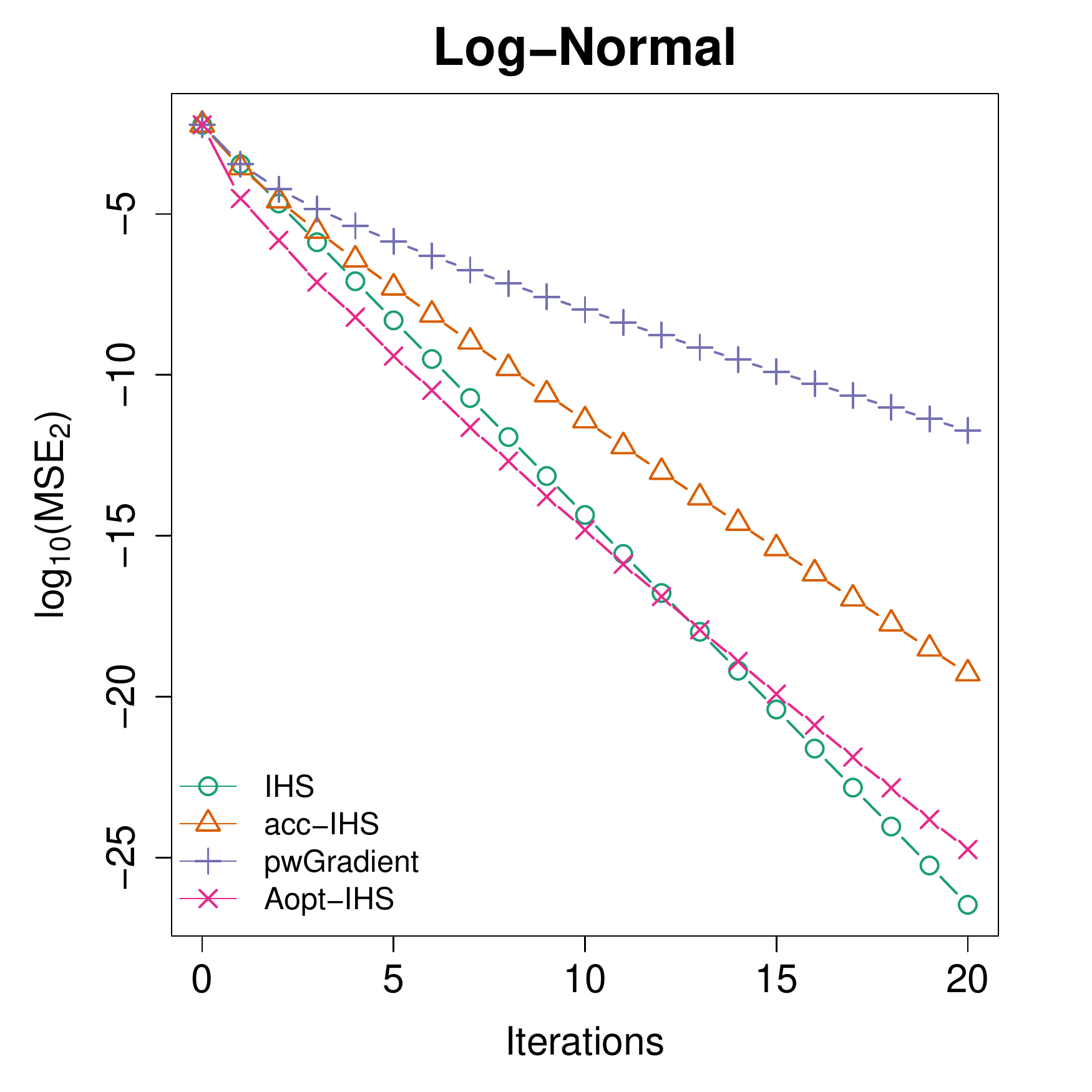}
	\end{minipage}}
	\subfigure[$t_2$]{
		\begin{minipage}[b]{0.4\linewidth}
			\includegraphics[width=\linewidth]{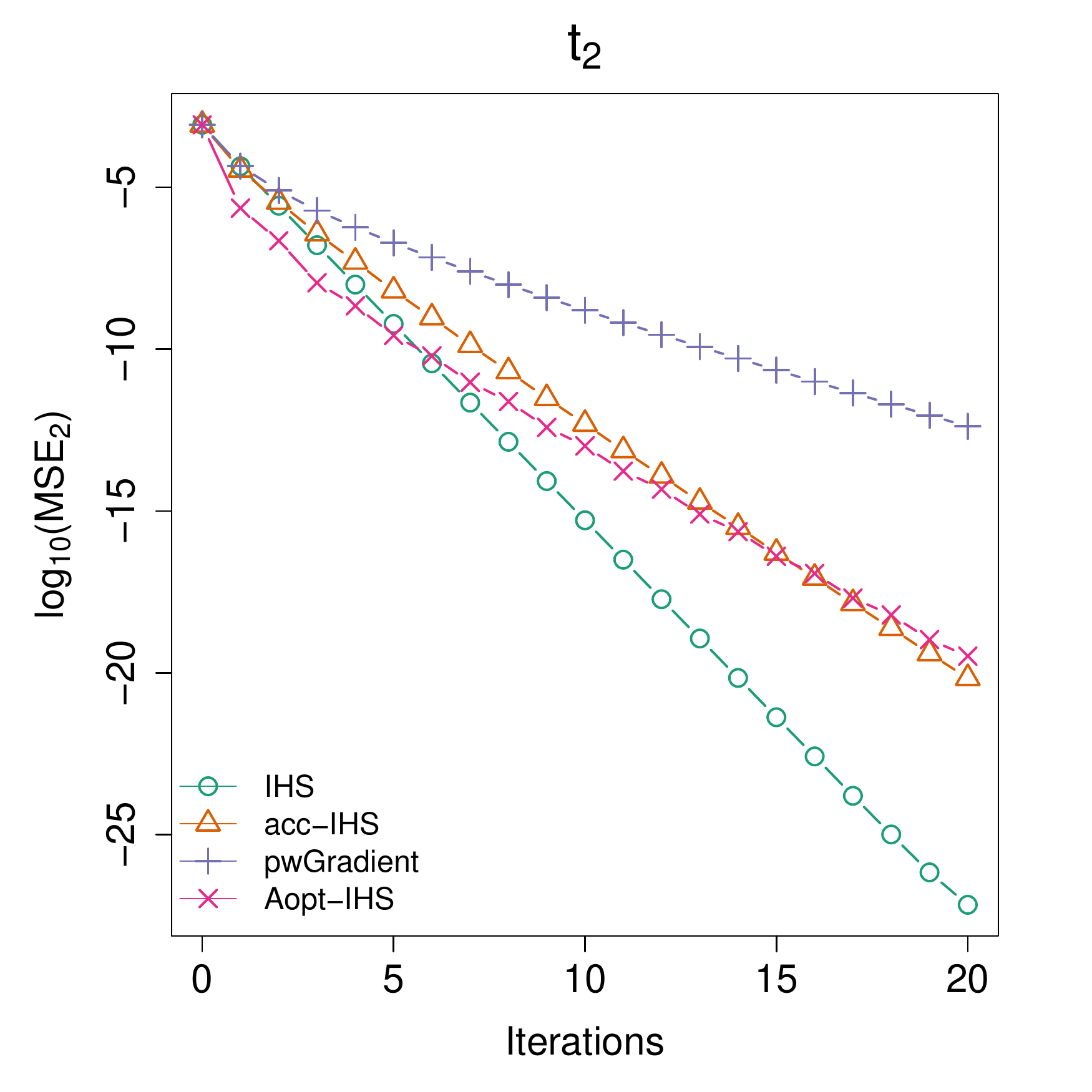}
	\end{minipage}}
	\subfigure[Mixture]{
		\begin{minipage}[b]{0.4\linewidth}
			\includegraphics[width=\linewidth]{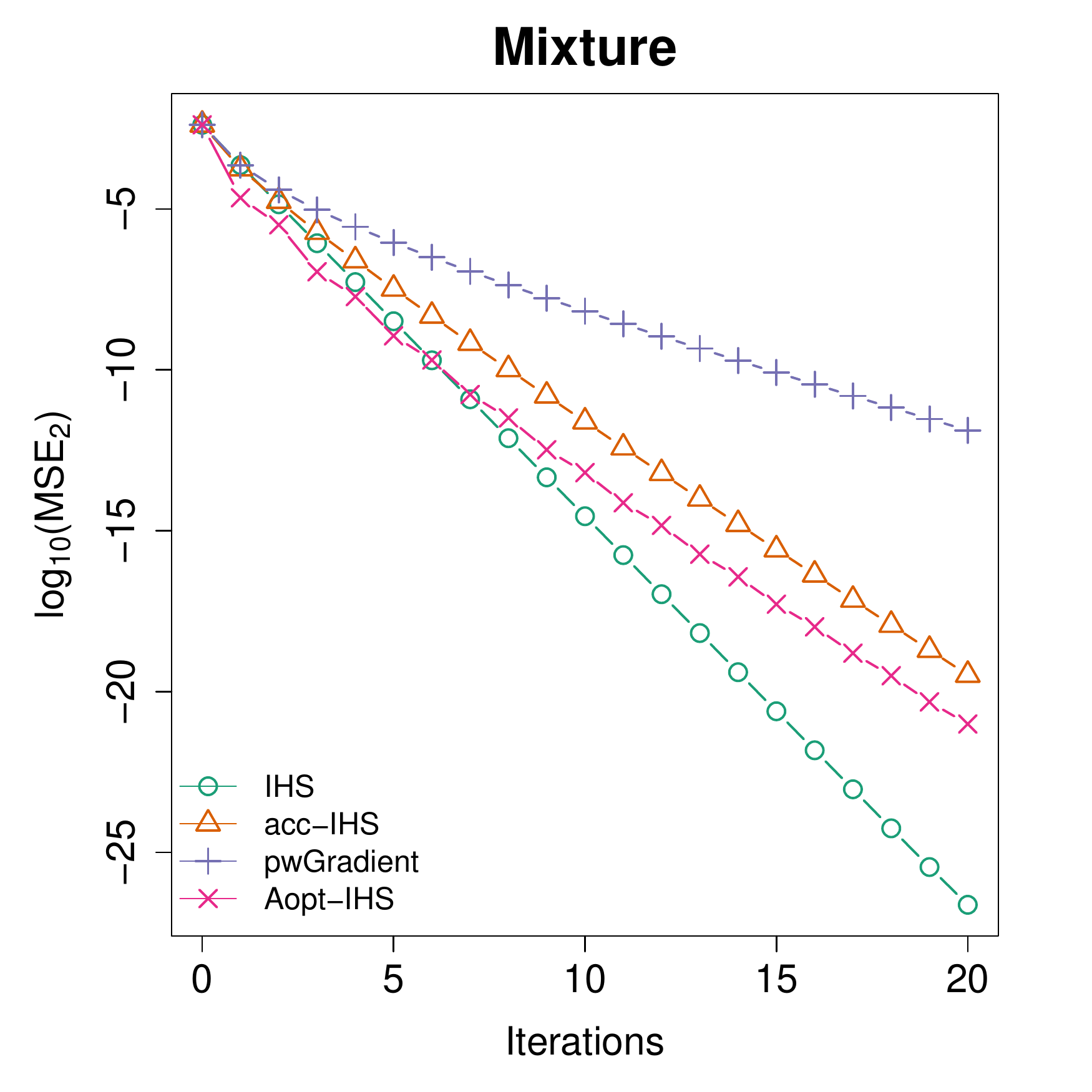}
	\end{minipage}}
	\begin{mycaption}
		Estimation error between the estimator and the LSE based on full data versus iteration number $N$ when $d=50$.
	\end{mycaption}
\end{figure}
 
\end{document}